\documentclass{article}

\usepackage{microtype}
\usepackage{graphicx}
\usepackage{subcaption}
\usepackage{booktabs} 

\usepackage{hyperref}

\usepackage[accepted]{icml2026}

\usepackage{amsmath}
\usepackage{amssymb}
\usepackage{mathtools}
\usepackage{amsthm}

\usepackage[capitalize,noabbrev]{cleveref}

\theoremstyle{plain}
\newtheorem{theorem}{Theorem}[section]

\newtheorem{corollary}[theorem]{Corollary}
\theoremstyle{definition}

\newtheorem{assumption}[theorem]{Assumption}
\theoremstyle{remark}

\usepackage[textsize=tiny]{todonotes}

\icmltitlerunning{Structure-Centric Graph Foundation Model via Geometric Bases}

\begin{document}

\twocolumn[
  \icmltitle{Structure-Centric Graph Foundation Model via Geometric Bases}



  \icmlsetsymbol{equal}{*}
  \icmlsetsymbol{corre}{$\dagger$}

  \begin{icmlauthorlist}
    \icmlauthor{Xiaodong He}{equal,yyy}
    \icmlauthor{Haolan He}{equal,yyy}
    \icmlauthor{Ruiyi Fang}{wu}
    \icmlauthor{Ming Sun}{yyy}
    \icmlauthor{Zhao Kang}{corre,yyy}
  \end{icmlauthorlist}

  \icmlaffiliation{yyy}{University of Electronic Science and Technology of China, Chengdu, Sichuan Province, China}
  \icmlaffiliation{wu}{Western University}

  \icmlcorrespondingauthor{Xiaodong He}{hexiaodong24@126.com}

  \icmlcorrespondingauthor{Zhao Kang}{zkang@uestc.edu.cn}
  

  \icmlkeywords{Machine Learning, ICML}

  \vskip 0.3in
]



\printAffiliationsAndNotice{
\icmlEqualContribution
\quad
{$\dagger$}Corresponding authors.
}  

\begin{abstract}
Graph foundation models (GFMs) seek transferable representations across graph domains but are limited by structural heterogeneity and incompatible node feature spaces.
We propose Structure-Centric Graph Foundation Models (SCGFM), which treat graph topology as the primary source of transferable knowledge.
Modeling graphs as metric measure spaces, SCGFM introduces learnable geometric bases that define a shared structural coordinate system.
Graphs are aligned to these bases via Gromov–Wasserstein distances, yielding structure-aligned latent representations that accommodate heterogeneous graph topologies.
To address feature incompatibility, SCGFM employs a structure-aware feature re-encoding mechanism that unifies node representations without assuming a fixed feature dimensionality or requiring dataset-specific preprocessing.
Experiments on graph- and node-level tasks demonstrate strong in-domain and cross-domain generalization, outperforming existing GFM approaches.
\end{abstract}

\section{Introduction}
\begin{figure}[ht]
    \centering
    \includegraphics[width=\linewidth]{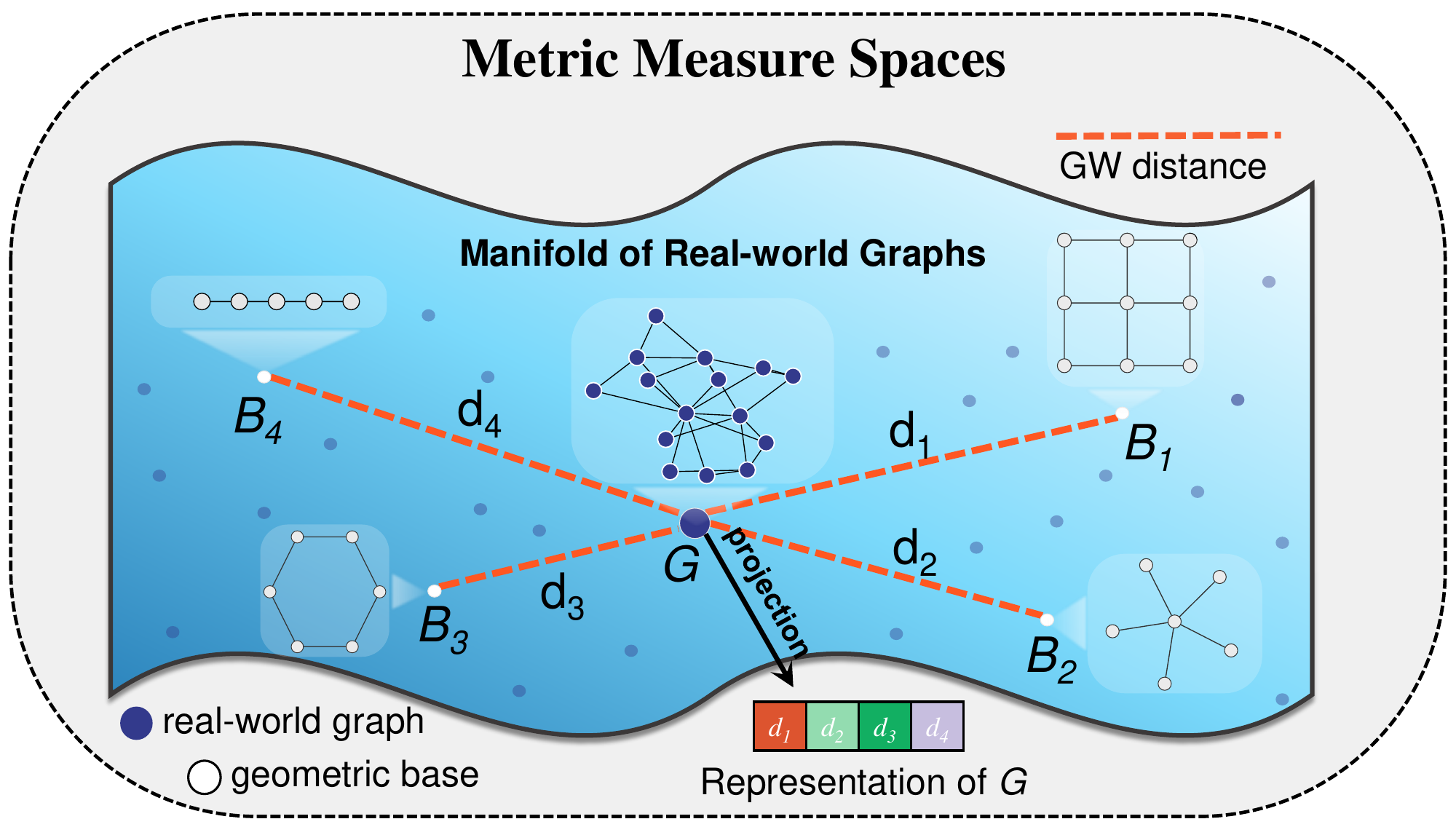}
    \caption{
     \textbf{Geometric perspective of SCGFM.}
  We view graph representation learning as a triangulation process in the space of metric measure spaces.
  Real-world graphs are assumed to lie on a low-dimensional manifold (blue region).
  Given an input graph $G$, its structural discrepancies to a shared set of geometric bases $\{B_k\}$ are measured using the Gromov--Wasserstein (GW) distance.
  These distances define a coordinate representation of $G$ as $\mathbf{q}=[d_1, d_2, \ldots]$, yielding a unified embedding that aligns heterogeneous graphs across domains.
    }
    \label{fig:cover_system}
\end{figure}

Foundation models have transformed natural language processing \cite{devlin-etal-2019-bert,NEURIPS2020_1457c0d6} and computer vision \cite{dosovitskiy2021image,he2022mae} through large-scale pretraining and strong cross-domain generalization.
These successes have motivated growing interest in Graph Foundation Models (GFMs), which aim to extend the same paradigm to graph-structured data.
Existing approaches largely follow two directions:
(1) augmenting graph models with large language models via prompting or adapters to inject linguistic priors \cite{ye-etal-2024-language,NEURIPS2023_622afc4e,zhao2023gimlet}, and
(2) pretraining graph neural networks on large graph corpora using contrastive or generative objectives \cite{wangmulti,yuan2025how}.
Despite promising progress, current GFMs still struggle to generalize reliably across graph datasets drawn from disparate domains and distributions.

A central challenge is the lack of a shared geometric reference for graphs \cite{liu2025gfm}.
Unlike text or images, graphs are relational objects defined only up to isomorphism and vary substantially across domains in topology, scale, and induced metrics.
Most existing GFM frameworks address this heterogeneity by enforcing a fixed node feature dimensionality \cite{Hu*2020Strategies,NEURIPS2022_5d4834a1}, typically via dataset-specific projections, padding, or dimensionality reduction, often implemented through feature adapters \cite{yuan2025how, shenwhen}.
While effective in practice, these strategies align feature spaces without aligning graph structure, causing learned representations to poorly reflect intrinsic graph similarity and limiting cross-domain transfer.

Recent work has explored graph tokenization \cite{pmlr-v267-chen25cf,ICLR2025_f2059277} as an alternative, discretizing graphs into symbolic tokens analogous to words in language models.
However, this paradigm is fundamentally misaligned with the geometric nature of graphs.
Graphs are non-Euclidean objects and permutation-invariant \cite{Bruna2013SpectralNA, fang2025benefits}, whereas tokenization typically imposes artificial orderings on inherently unordered structures \cite{pmlr-v80-you18a}.
As a result, token-based representations often violate permutation invariance and fail to capture intrinsic graph geometry, leading to limited robustness and generalization \cite{Bronstein2021GeometricDL,pmlr-v80-jin18a}.

In this work, we adopt a geometric perspective and model graphs as \textbf{metric measure spaces} (mm-spaces), in which structure is defined independently of node identities and feature semantics.
Motivated by results from metric geometry, we posit that real-world graphs lie within a structured and bounded subset of the space of mm-spaces.
Under this assumption, graphs can be represented by their Gromov–Wasserstein (GW) distances \cite{memoli2011gromov} to a finite set of canonical geometric patterns.
This distance-based projection induces a shared, continuous representation space that naturally accommodates heterogeneous graph structures (Figure~\ref{fig:cover_system}).

Building on this insight, we propose \textbf{Structure-Centric Graph Foundation Models (SCGFM)} based on learnable geometric bases.
Each base encodes a canonical graph structure with a fixed number of nodes and trainable edge weights, collectively forming a shared \textbf{structural coordinate system}.
Input graphs are aligned to this system via structure-preserving GW mappings, yielding unified latent representations that explicitly encode graph topology.
Unlike prototype or dictionary learning methods that operate primarily in feature space \cite{10.5555/3618408.3620116,pmlr-v139-vincent-cuaz21a}, our geometric bases are optimized to capture intrinsic structural variability across graphs.

Beyond structural heterogeneity, incompatibility among node feature spaces presents an additional  \cite{pmlr-v267-wang25an, NEURIPS2024_380a0b16, li2024pc}.
Rather than enforcing a fixed feature dimensionality, SCGFM leverages structural alignment to induce feature unification: node features are aggregated and re-encoded through their correspondence to geometric bases.
This produces unified feature representations anchored to structure, while remaining agnostic to the original feature dimensionality and semantics.

Together, these components yield a unified and scalable GFM that is both structure-centric and feature-flexible.
SCGFM\footnote{Code is available at: \url{https://github.com/Xd-He/SCGFM}} can be pretrained on graphs from diverse domains and transferred across datasets without architectural modification or dataset-specific preprocessing, offering a principled alternative to graph tokenization and domain-dependent GFMs.
Our contributions are summarized as follows:
\begin{itemize}
    \item We propose a structure-centric graph foundation model that enables transfer across heterogeneous graph datasets without dataset-specific architectures or feature preprocessing.
    \item We show that the learned geometric bases induce a geometry-aligned latent space whose distances correlate strongly with intrinsic Gromov–Wasserstein graph geometry.
    \item Extensive few-shot evaluations on graph- and node-level tasks demonstrate strong in-domain and cross-domain generalization.
\end{itemize}

\begin{figure*}[!t]
  \centering
  \includegraphics[width=0.9\linewidth]{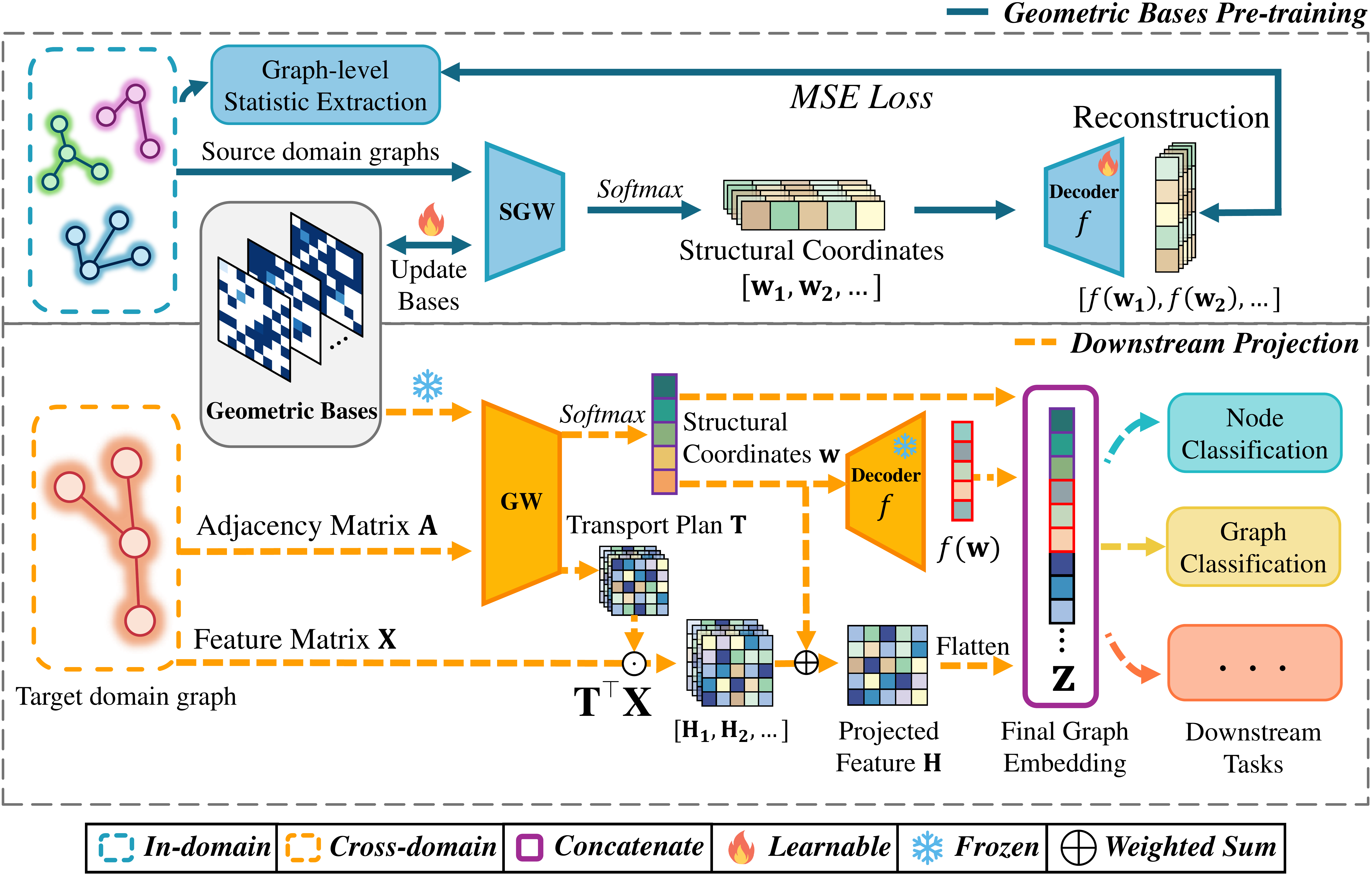}
  \caption{
      \textbf{Overall framework of SCGFM}.
      In \textbf{Geometric Bases Pre-training}, trainable geometric bases are optimized via the Sliced Gromov-Wasserstein (SGW) distance to map each source-domain graph to a structural coordinate.
      In \textbf{Downstream Projection}, a target graph is matched to the bases to obtain a structural coordinate vector and a GW transport plan for projecting the feature matrix.
      The final graph embedding is the concatenation of the structural coordinate $\mathbf{\textbf{w}}$, decoder output $f(\mathbf{\textbf{w}})$, and projected feature $\mathbf{\textbf{H}}$ for downstream tasks.
  }
    \label{fig:framework}
\end{figure*}

\section{Related Work}

Graph representation learning has evolved from self-supervised pre-training to GFMs designed for broad transferability across datasets and tasks \cite{liu2023graphssl,xie2023sslgnn,fang2026graph,liu2025gfm, pmlr-v202-pan23b}.
However, \emph{cross-domain few-shot} transfer remains challenging due to (i) substantial \emph{topological shifts} and (ii) \emph{heterogeneous} or missing node features, which undermine feature-centric approaches.
\subsection{Self-Supervised Graph Pre-training}
Early methods focused on contrastive views (e.g., DGI \cite{velickovic2018deep}, GraphCL \cite{you2020graph}), while later approaches minimized handcrafted augmentations (GraphACL \cite{xiao2023simple}) or adopted scalable masked modeling (GraphMAE2 \cite{hou2023graphmae2}, S2GAE \cite{tan2023s2gae}).

Recent works improve generalization via unified objectives (GALE \cite{pmlr-v267-wang25ez}) and structure-aware curriculum masking (CurMGAE \cite{pmlr-v267-li25ct} and SAGA \cite{fangsaga}).
While effective in-domain, these methods typically assume compatible feature spaces within a shared embedding domain. Consequently, they often struggle with topology-driven shifts or inconsistent features under few-shot settings, limiting cross-dataset transferability.

\subsection{GFMs Across Domains and Tasks}
GFMs unify learning via shared interfaces and large-scale pre-training \cite{liu2025gfm,wang2025graph}.
Interface-centric designs align tasks using language or structured prompts (OFA \cite{liuone}, GIT \cite{wangtowards}, UniGraph \cite{10.1145/3690624.3709277}, SIGOOD \cite{wang2026subtle}). 
Recent research emphasizes cross-domain robustness: MDGFM \cite{wangmulti} and SAMGPT \cite{Yu2025SAMGPTTG} focus on topology and structure alignment, while BRIDGE \cite{yuan2025how} explores selective knowledge assembly.
Retrieval-augmented models like RAG4GFM \cite{NEURIPS2025_0c3ce12a} and RAG-GFM \cite{yuan2026rag} further enhance inference reliability.
However, most GFMs still rely on feature alignment or prompts that fail when target graphs exhibit strong topological shifts and unreliable features.

\subsection{GW-based graph methods.}
Several recent works have applied optimal transport to graph-structured data. Vayer et al.~\cite{pmlr-v97-titouan19a} introduced Fused Gromov--Wasserstein (FGW) to jointly compare node features and graph structures, while Chen et al.~\cite{pmlr-v202-chen23ak} employed a GW geometric perspective for spectrum-preserving graph coarsening. These studies highlight the usefulness of GW-type distances for graph comparison and representation learning. In contrast, SCGFM learns a set of geometric bases that define a shared structural coordinate system.

\section{The Proposed Method}

We propose \textbf{SCGFM} that learns a finite set of geometric bases to represent arbitrary graphs within a unified metric latent space.
It consists of two stages: (1) \textbf{Geometric Bases Pre-training}, and (2) \textbf{Downstream Projection}.
An overview of the framework is shown in Figure~\ref{fig:framework}.

\subsection{Geometric Motivation: Finite Covering of Graph Space}
We regard the set of all finite graphs, endowed with normalized metrics and measures, inducing elements in a common mm-space, and denote by $\mathcal{X}$ the resulting space equipped with the GW distance $d_{GW}$.
Under this view, graph representation learning amounts to learning over the space $(\mathcal{X}, d_{GW})$, where each graph is represented by its relational geometry rather than explicit node correspondences.

\begin{assumption}[\textbf{Total Boundedness}]
  We assume real-world graphs lie in a totally bounded subset $\mathcal K \subset \mathcal X$. 
  By Gromov's compactness theorem for mm-spaces \cite{Gromov1999Metric}:
  \begin{itemize}
    \item Total boundedness is equivalent to the existence of finite $\epsilon-$covers under the GW metric.
    \item Thus, for any tolerance $\epsilon >0$, there exist prototypes $\left\{C_1,\ldots,C_K \right\}$ such that each graph $G \in \mathcal{K}$ satisfies
    \begin{equation}
      d_{GW}(G,C_k) < \epsilon \quad \text{for some }k.
    \end{equation}
  \end{itemize}
  \label{ass:boundedness}
\end{assumption}
To instantiate this theoretical guarantee, we propose learning a discrete set of geometric bases that approximates the optimal $\epsilon$-cover directly from data.
Unlike task-specific templates, these learnable prototypes induce a shared geometric coordinate system, in which any graph can be represented through its transport discrepancies to the bases. We next specify the parameterization of these bases and show how it ensures that each prototype constitutes a valid mm-space.

\subsection{Geometric Bases Pre-training}
Each geometric base is defined as a finite mm-space:
\begin{equation}
  B_k = ([M],d_k,\mu_k),
\end{equation}
where the metric component $d_k$ is represented by a learnable matrix
$\mathbf{B}_k \in [0,1]^{M \times M}$.
$\mathbf{B}_k$ is symmetric, hollow, and bounded, and serves as the concrete parameterization of the geometric base.
Following common practice in GW-based learning, we do not enforce the triangle inequality, as pseudo-metrics remain valid distance kernels.
The measure $\mu_k$ is fixed to a uniform distribution over the $M$ points, which avoids introducing unnecessary degrees of freedom and ensures well-defined GW couplings.

\textbf{Graph as a Metric Measure Space.}
We denote each input graph as $G=(\mathcal{V}, \mathcal{E}, \mathbf{A}, \mathbf{X})$, with node set $\mathcal{V}$, edge set $\mathcal{E}$, adjacency matrix $\mathbf{A}$, and node features $\mathbf{X}\in\mathbb{R}^{N\times F}$ ($N=|\mathcal{V}|$, $F$ is the feature dimension). We then convert $G$ into a mm-space:
\begin{equation}
  \mathcal{G} = (\mathcal{V},d_G,\mu_G),
\end{equation}
where $d_G$ creates the structure of graph, represented here by $\mathbf{A}$. 
$\mu_G$ is a \textbf{degree-based measure}:
\begin{equation} \label{eq:mu_degree}
  \mu_G(i) = \frac{\text{deg}(i)}{\sum_{v \in V}\text{deg}(v)}.
\end{equation}
where $\text{deg}(i) $ is the degree of i-th node.

\textbf{GW-based Structural Alignment.}
For each graph $G_i$, we compute its GW discrepancy with each base:
\begin{equation} \label{eq:gw_distance}
  \delta_k = d_{GW}(\mathbf{A}_i, \mathbf{B}_k),
\end{equation}
approximated using \textbf{Sliced Gromov-Wasserstein (SGW)} \cite{vayer2019sliced}. 
Crucially, by projecting the metric structure onto 1D slices, SGW reduces the computational bottleneck from the prohibitive cubic complexity $\mathcal{O}(N^3)$ of exact solvers to a \textbf{quasi-linear $\mathcal{O}(N \log N)$} cost.
This efficiency is fundamental to our method's scalability.

\textbf{Linear Surrogate of GW Barycenter.}
Computing the true GW barycenter
\begin{equation}
  \text{arg }\underset{B}{\operatorname{min }} \mathbb{ E}_{\mathcal{G}} \left[ d_{GW}(\mathcal{G}, B)\right],
\end{equation} 
is intractable because barycenter computation requires nested Optimal Transport (OT) optimization \cite{peyre2016gromov}. 
We thus adopt a \textbf{linear surrogate model}:
\begin{equation} \label{eq:surrogate_B}
  \widetilde{\mathbf{B}}(G) = \sum_{k=1}^{K} w_k \mathbf{B}_k,
\end{equation}
where $K$ denotes the number of geometric bases. The weight vector $\mathbf{w}$, which serves as the \textbf{structural coordinates} of $G$, is computed based on structural similarity:
\begin{equation} \label{eq:weights}
  w_k = \frac{\text{exp}(-\delta_k/ \tau)}{\sum_j \text{exp}(-\delta_j /\tau)}. 
\end{equation}
where $\tau$ is a temperature hyperparameter.

\begin{theorem}[\textbf{Stability of Structural Coordinates}] \label{the:stability}
  Let $G$ and $G'$ be two graphs with bounded GW distance $d_{GW}(\mathcal{G},\mathcal{G}') \leq \eta,0 \leq \eta < \infty$.
  The structural coordinates $\mathbf{w}$ and $\mathbf{w}'$ computed via Eq.~\ref{eq:weights} is Lipschitz continuous with respect to the GW distance:
  \begin{equation}
    \Vert \mathbf{w} - \mathbf{w}' \Vert_2 \leq L_w \cdot \eta,
  \end{equation} 
\end{theorem}
where $L_w = \frac{\sqrt{K}L_{sm}}{\tau}$ is the Lipschitz constant and $L_{sm}$ is the Lipschitz constant of the softmax function over a bounded domain.
We provide the detailed proof in Appendix~\ref{proof:w_stability}.

\textbf{Reconstruction of Graph-level Structure.}
Linear combination is not a literal mm-space barycenter, but an efficient approximation of a GW barycenter expansion through a finite basis dictionary.
Consequently, the structural reconstruction loss is formulated as:
\begin{equation} \label{eq:gw_loss}
  \mathcal{L}_{\text{gw}} = \mathbb{E}_{G} \left[d_{GW}(\mathbf{A}, \widetilde{\mathbf{B}}(G)) \right].
\end{equation}

\textbf{Reconstruction of Graph-level Statistics.}
To ensure that structural coordinates preserve key graph-level statistics, we decode $\mathbf{w}$ using a Multi-Layer Perceptron (MLP) decoder $f(\cdot) \in \mathbb{R}^r$ . 
The reconstruction target includes degree histogram, clustering histogram, and log-scaled motifs (triangles and short cycles).
The reconstruction loss is:
\begin{equation} \label{eq:rec_loss}
  \mathcal{L}_{\text{rec}} = \text{MSE}(\text{FE}(G),f(\mathbf{w})),
\end{equation}
where $\text{FE}(\cdot)$ denotes the feature-extraction operator that computes graph-level statistics. 

\begin{corollary}[\textbf{Stability of Statistics Reconstruction}]
  Following Theorem~\ref{the:stability}, the reconstructed graph-level statistics $f(\mathbf{w})$ is Lipschitz continuous with respect to the GW distance:
  \begin{equation}
    \Vert f(\mathbf{w}) - f(\mathbf{w}') \Vert 
   _2 \leq L_f \Vert \mathbf{w} - \mathbf{w}'\Vert _2 
    \leq L_f L_w \cdot \eta.
  \end{equation} 
\end{corollary}
where $L_f$ is the Lipschitz constant of the decoder $f(\cdot)$.
This guarantees that structurally similar graphs produce similar reconstructed statistical descriptors, thereby preserving intrinsic geometric semantics.

\textbf{Regularization of Structural Diversity.}
To prevent bases from collapsing, we enforce diversity by minimizing the pairwise similarity of the learnable geometric bases $\mathbf{B}_k$:
\begin{equation} 
  \mathcal L_{\text{div}} = \frac{1}{|\mathcal P|} \sum_{(i,j)\in \mathcal P} \text{max}(0, m- \Vert\mathbf{B}_i - \mathbf{B}_j \Vert_{\mathrm{F}}),
  \label{eq:div_loss}
\end{equation}
where $\mathcal{P} = \left\{ (i,j) |1 \leq i <j \leq K \right\}$ represents the set of all unique pairs of bases and $\Vert \cdot \Vert_{\mathrm{F}} $ denotes the Frobenius norm.
This loss enforces a minimum separation of 
$m$ between any pair of bases, ensuring that the geometric bases span a broad region of the structural space.

\textbf{Total Objective.} 
The final training objective is a weighted sum of the reconstruction loss and the regularization terms:
\begin{equation}
  \mathcal{L}_{\text{total}} = \mathcal{L}_{\text{gw}} + \alpha \mathcal{L}_{\text{rec}} + \beta \mathcal{L}_{\text{div}},
  \label{eq:total_loss}
\end{equation}
where $\alpha$ and $\beta$ are trade-off hyperparameters. 

\subsection{Downstream Projection}
To derive a unified graph representation using pre-trained geometric bases,
we project each input graph into a shared geometric coordinate system
and aggregate the projected structural and feature information.

\textbf{Coordinate Projection.}
For a given input graph $G_i$, we directly compute its structural coordinates $\mathbf{w}$ with respect to the pre-trained geometric bases $B$, strictly following the definition in Eq.~\ref{eq:weights}.
This maps the input graph to the shared geometric coordinate system.

\textbf{Feature Projection.}
Using the OT plan $\mathbf{T}_{ik} \in \mathbb{R}^{N \times M}$ obtained during the GW distance computation between $(G_i, B_k)$ to project node features onto the geometric bases.
To adhere to the injectivity requirements for multiset representations \cite{xu2018gin}, we adopt summation aggregation.
Specifically, since $\mathbf{T}_{ik}$ operates as a normalized measure (inducing averaging), we rescale the projection by $N$ to restore signal magnitude:
\begin{equation} \label{eq:project_h}
  \begin{split}
    \mathbf{H}_k &= N *\mathbf{T}_{ik}^{\top} \mathbf{X}_i \in \mathbb{R}^{M \times F}, \\
    \mathbf{H}(G_i) &= \sum_{k=1}^{K} w_k \mathbf{H}_k \in \mathbb{R}^{M \times F}.
  \end{split}
\end{equation}
Finally, we construct the final graph embedding $\mathbf{z}(G_i)$ by concatenating the structural coordinates  $\mathbf{w}$, reconstruction of graph-level statistics and the projected features:
\begin{equation} \label{eq:graph_embedding}
  \mathbf{z}(G_i) = [\mathbf{w} \mathbin{\|} f(\mathbf{w}) \mathbin{\|} \text{vec}(\mathbf{H}(G_i))] \in \mathbb{R}^{K+r+MF}.
\end{equation}
Note that we use the frozen pre-trained decoder $f(\cdot)$.

\section{Experiments}
In this section, we evaluate the performance of SCGFM on few-shot classification tasks. 
Specifically, we aim to answer the following research questions:
(1) \textbf{RQ1 (Performance):} Does SCGFM achieve superior generalization capabilities in few-shot learning, encompassing both in-domain and cross-domain transfers?
(2) \textbf{RQ2 (Ablation and Scalability):} How do individual components and key design choices contribute to the effectiveness of SCGFM, and can the proposed design remain scalable to large graphs?
(3) \textbf{RQ3 (Interpretability):} What do the learned geometric bases look like, and do they possess meaningful physical or topological interpretations?
(4) \textbf{RQ4 (Mapping Quality):} How well does the learned latent space preserve the intrinsic geometric properties of the input graphs?

\newcommand{\res}[2]{$#1_{\pm #2}$}
\newcommand{\resb}[2]{$\mathbf{#1}_{\pm \mathbf{#2}}$}
\begin{table*}[!ht]
    \caption{ Accuracy (\% $\pm$ standard deviation for 50 runs ) of 5-shot graph classification.
    OOD is Out-of-Domain.
    COL-3 = COLORS-3. 
    IMDB-B = IMDB-BINARY. 
    S1 = COX2 + NCI1 + BZR. 
    S2 = COLLAB + IMDB-B. 
    The best results are shown in \textbf{bold} and the runner-ups are \underline{underlined}.}
    \label{tab:few_shot_graph}
    \begin{center}
        \begin{small}
                \resizebox{\textwidth}{!}{
                    \begin{tabular}{lccccccccccc}
                        \toprule
                        Method 
                            & \multicolumn{3}{c}{In-Domain} 
                            & \multicolumn{3}{c}{Trained on S1 (OOD)} 
                            & \multicolumn{4}{c}{Trained on S2 (OOD)} \\
                        \cmidrule(lr){2-4} \cmidrule(lr){5-7}  \cmidrule(lr){8-11}
                                          & COX2               & NCI1               & BZR                & COLLAB             & IMDB-B             & COL-3               & COL-3             & COX2                & NCI1              & BZR               & Avg.\\
                        \midrule
                        GCN \cite{kipf2017gcn}& \res{49.84}{4.24}  & \res{51.85}{4.46}  & \res{54.41}{6.93}  & \res{46.35}{6.72}  & \res{51.21}{8.59}  & \res{9.53}{1.06}   & \res{9.37}{0.95}   & \res{50.33}{3.74}  & \res{51.75}{5.52}  & \res{57.66}{5.45}               &43.23     \\
                        GAT \cite{veličković2018gat} & \res{52.05}{4.29}  & \res{52.36}{5.02}  & \res{55.26}{5.39}  & \res{32.82}{3.12}  & \res{49.60}{2.92}  & \res{9.12}{0.83}   & \res{9.38}{0.98}   & \res{51.42}{5.01}  & \res{51.98}{5.79}  & \res{55.29}{6.88}                            &41.93     \\
                        GIN \cite{xu2018gin} & \res{54.31}{6.36}  & \res{52.95}{4.77}  &  \res{51.29}{4.67} & \res{58.11}{4.48}  & \res{55.36}{8.28}  & \res{9.25}{1.10}   & \res{9.12}{1.03}   & \resb{55.16}{7.05}  & \res{51.60}{5.64}  & \res{51.44}{5.28}                            &44.85      \\
                        \midrule
                        DGI \cite{velickovic2018deep} & \res{49.34}{5.34}  & \res{50.86}{4.85}  & \res{50.30}{6.10}  & \res{51.77}{7.06}  & \res{54.26}{9.18}  & \res{9.67}{1.17}   & \res{9.77}{1.21}   & \res{49.40}{5.18}  & \res{50.60}{4.69 } & \res{49.98}{6.94}                            &42.60     \\
                        GraphCL \cite{you2020graph}  & \res{54.68}{7.47}  & \underline{\res{57.22}{9.75}}  & \underline{\res{60.28}{8.73}}  & \res{41.95}{7.48}  & \res{48.26}{4.78}  & \res{9.17}{1.43}   & \res{9.25}{2.19}   & \res{50.14}{6.06}  & \underline{\res{56.78}{10.23}} & \res{55.50}{6.28}    &43.90     \\
                        GraphMAE \cite{hou2022graphmae}  & \underline{\res{55.16}{11.46}} & \res{52.27}{4.06}  & \res{56.58}{12.62}  & \res{36.76}{7.53}  & \underline{\res{56.35}{6.48}}  & \res{13.74}{1.02}  & \res{9.07}{0.49}   & \underline{\res{55.09}{9.33}}  & \res{50.30}{2.89}  & \res{54.52}{10.65}  &43.98     \\
                        GraphACL\cite{xiao2023simple}    & \res{51.70}{6.91}  & \res{50.90}{5.16}  & \res{57.72}{8.70}  & \resb{63.59}{3.48}  & \res{53.79}{6.55}  & \res{11.91}{1.37}  & \res{9.13}{0.89}   & \res{51.00}{6.68}  & \res{50.44}{5.43}  & \res{51.50}{4.85}                           &45.17     \\
                        S2GAE \cite{tan2023s2gae}      & \res{53.19}{5.11}  & \res{52.14}{4.95}  & \res{55.83}{6.96}  & \res{54.79}{7.14}  & \res{50.48}{4.03}  & \underline{\res{14.47}{1.66}}  & \underline{\res{15.65}{2.02}}  & \res{52.51}{4.53}  & \res{51.75}{4.92}  & \res{56.10}{6.84}    &\underline{45.70}      \\
                        \midrule
                        GCOPE \cite{sun2024gcope}   & \res{53.20}{9.06}  & \res{51.75}{10.31}  & \res{56.90}{8.92}  & \res{61.07}{7.61}  & \res{51.60}{9.07}  & \res{9.16}{1.81}   & \res{8.95}{1.58}   & \res{50.60}{6.55}  & \res{51.95}{7.47}  & \res{58.10}{10.22}                            &45.32    \\
                        RiemannGFM \cite{sun2025riemanngfm} & \res{50.78}{1.41}  & \res{50.84}{2.18}  & \res{50.60}{1.75}  & \res{61.52}{2.64}  & \res{56.32}{3.06}  & \res{8.88}{1.00}   & \res{10.36}{0.87}  & \res{50.36}{1.77}  & \res{51.28}{2.19}  & \underline{\res{58.24}{3.04}}                           &44.92     \\
                        GIT \cite{wangtowards}      & \res{51.54}{5.31}  & \res{50.70}{6.52}  & \res{51.86}{8.65}  & \res{49.32}{12.17} & \res{54.62}{8.44}  & \res{9.37}{1.09}   & \res{9.04}{ 1.45}  & \res{49.30}{4.08}  & \res{52.12}{6.72}  & \res{52.46}{6.88}                &43.03      \\
                        \midrule
                        SCGFM (Ours)              & \resb{56.60}{6.12}  & \resb{58.00}{8.90}  & \resb{60.80}{6.60}  & \underline{\res{61.93}{6.11}}  & \resb{56.62}{9.00}  & \resb{26.31}{2.99}  & \resb{26.97}{3.08} & \res{53.80}{5.61}  & \resb{57.50}{7.32}  & \resb{58.50}{6.28} &\textbf{51.70}     \\
                        \bottomrule
                    \end{tabular}
                }
        \end{small}
    \end{center}
    \vskip -0.1in
\end{table*}

\begin{table*}[!ht]
    \caption{ Accuracy (\% $\pm$ standard deviation for 50 runs ) of 5-shot node classification.
    N1 = Cora + CiteSeer + PubMed.
    N2 = Photo + Computers.
    The best results are shown in \textbf{bold} and the runner-ups are \underline{underlined}.
    }
    \label{tab:few_shot_node}
    \begin{center}
        \begin{small}
                \resizebox{\textwidth}{!}{
                    \begin{tabular}{lcccccccccc|c}
                        \toprule
                        Method & \multicolumn{3}{c}{In-Domain} & \multicolumn{3}{c}{Trained on N1 (OOD)} & \multicolumn{4}{c}{Trained on N2 (OOD)} \\
                        \cmidrule(lr){2-4} \cmidrule(lr){5-7}  \cmidrule(lr){8-11}
                                          & Cora                & CiteSeer          & PubMed             & Photo              & Computers          & Reddit             & Cora                & CiteSeer            & PubMed             & Reddit               & Avg.\\
                        \midrule
                        GCN \cite{kipf2017gcn}     & \res{36.01}{3.33}  & \res{27.21}{3.88}  & \res{44.08}{5.32}  & \res{41.21}{4.20}  & \res{28.00}{3.35}  & \res{18.79}{1.26}  & \res{47.13}{2.26}   & \res{30.19}{4.18}  & \res{42.57}{5.76}  & \res{44.54}{1.36}       &   35.97  \\
                        GAT \cite{veličković2018gat}   & \res{39.47}{3.89}  & \res{30.65}{4.21}  & \res{45.68}{5.58}  & \res{35.60}{3.39}  & \res{31.01}{3.19}  & \res{13.87}{1.65}  & \res{48.38}{3.18}   & \res{35.38}{4.31}  & \underline{\res{52.56}{6.89}}  & \res{20.60}{1.59}       &  35.32   \\
                        GIN \cite{xu2018gin}    & \res{39.23}{3.70}  & \res{24.60}{4.71}  & \res{41.48}{7.10} & \res{39.48}{4.84}  & \res{32.52}{2.89}  & \res{19.98}{1.79}   & \res{32.55}{3.01}   & \res{25.68}{3.84}  & \res{34.77}{5.65}  & \res{22.26}{1.49}       &   31.26    \\
                        \midrule
                        DGI \cite{velickovic2018deep}   & \res{54.61}{6.75} & \res{28.50}{6.42} & \res{54.60}{7.86} & \res{58.44}{5.56} & \res{42.18}{4.54} & \res{42.14}{2.38} & \res{50.67}{6.00} & \res{25.13}{5.28} & \res{51.07}{8.33} & \res{62.60}{2.55} &   47.00  \\
                        GraphCL \cite{you2020graph}     & \res{60.82}{6.38}  & \underline{\res{35.50}{5.81}}   & \res{57.19}{8.72}  & \res{50.95}{3.12}  & \res{30.44}{3.93}  & \res{67.75}{1.58}   & \underline{\res{59.31}{3.02}}   & \underline{\res{38.53}{4.47}}  & \res{43.17}{5.57}  & \underline{\res{67.56}{1.15}}     &  \underline{51.12}    \\
                        GraphMAE \cite{hou2022graphmae}    & \res{38.50}{5.38} & \res{34.32}{4.94} & \res{36.83}{4.36} & \res{31.57}{6.45} & \res{31.86}{6.71} & \res{13.39}{1.12} & \res{34.59}{4.33} & \res{28.45}{4.51} & \res{34.55}{3.71} & \res{32.21}{1.62} &   31.63  \\
                        GraphACL\cite{xiao2023simple}   & \res{48.00}{5.44} & \res{34.02}{5.86} & \res{35.70}{7.54} & \res{59.51}{4.94} & \res{51.48}{4.85} & \res{44.91}{2.22} & \res{34.79}{4.57} & \res{26.00}{5.69} & \res{35.17}{7.02} & \res{45.94}{2.23} &  41.56   \\
                        S2GAE \cite{tan2023s2gae}  &\underline{\res{66.83}{5.53}}  &\res{32.55}{6.38} &\underline{\res{58.90}{8.35}} & \underline{\res{61.13}{4.66}} &\resb{65.59}{3.96}  &\underline{\res{75.28}{1.65}}  &\res{33.73}{5.21}  &\res{23.73}{5.21}  &\res{39.77}{7.91}  &\res{40.95}{1.94}    &  49.87   \\
                        \midrule
                        GCOPE \cite{sun2024gcope}       & \res{47.84}{4.79}  & \res{30.13}{6.00}  & \res{43.23}{6.54}  & \res{60.92}{4.54}  & \res{53.70}{3.90}  & \res{55.72}{2.02}   & \res{43.43}{6.05}   & \res{26.08}{4.82}  & \res{42.23}{7.74}  & \res{23.00}{2.36}      &  42.63   \\
                        RiemannGFM \cite{sun2025riemanngfm} & \res{51.09}{1.10} & \res{31.95}{1.44} & \res{40.96}{2.25} & \res{48.97}{1.10} & \res{31.44}{1.05} & \res{33.51}{0.41} & \res{46.27}{1.23} & \res{31.61}{1.17} & \res{39.07}{1.59} & \res{18.65}{0.28} &   37.35  \\
                        GIT \cite{wangtowards}         & \res{32.01}{5.39} & \res{28.75}{4.71} & \res{36.83}{4.36} & \res{42.86}{6.20} & \res{35.76}{7.40} & \res{33.50}{5.76} & \res{34.91}{4.68} & \res{26.35}{4.36} & \res{34.65}{3.21} & \res{32.90}{5.57} &   33.85  \\
                        \midrule
                        SCGFM (Ours)              & \resb{70.54}{3.27}  & \resb{43.83}{4.26}  & \resb{68.19}{5.30}  & \resb{68.30}{3.93}  & \underline{\res{57.35}{4.22}}  & \resb{84.17}{1.23}  & \resb{70.55}{3.28} & \resb{43.82}{4.26}  & \resb{68.15}{5.26}  & \resb{84.16}{1.23} &  \textbf{65.90}  \\
                        \bottomrule
                    \end{tabular}
                }
        \end{small}
    \end{center}
    \vskip -0.1in
\end{table*}
\subsection{Experimental Setup}


\textbf{Datasets and Tasks.}
We evaluate cross-domain generalization on \textbf{12 datasets} from \textbf{5 domains}, covering both \textit{graph classification} and \textit{node classification}.
Graph classification includes COX2, NCI1, BZR (\textbf{Molecules}), COLLAB, IMDB-BINARY (\textbf{Social Networks}), and COLORS-3 (\textbf{Synthetic}).
Node classification includes Cora, CiteSeer, PubMed (\textbf{Academic}), Photo, Computers (\textbf{E-commerce}), and Reddit (\textbf{Social Networks}).
Dataset statistics are provided in Appendix~\ref{datasets}.

\textbf{Evaluation Protocol.}
We adopt a unified few-shot transfer learning paradigm:
models are pre-trained on the source domain and directly evaluated on target domains with a \textbf{fixed encoder}.
Downstream performance is measured using \textbf{Prototypical Network (PN)} for classification.

\textbf{Node Classification.}
We formulate node classification as a graph classification task via node-centric subgraph sampling. 
Adopting the \textbf{Personalized Page Rank (PPR)} sampling strategy \cite{Zeng2020GraphSAINT}, we extract subgraphs (size $\leq$ 100) centered on target nodes, which inherit the central node's label.
Details are deferred to Appendix~\ref{ap:experiment_setup}.

\textbf{Baselines.}
We compare SCGFM with:
(1) \textbf{Vanilla GNNs:} GCN~\cite{kipf2017gcn}, GIN~\cite{xu2018gin}, GAT~\cite{veličković2018gat};
(2) \textbf{Self-supervised methods:} DGI~\cite{velickovic2018deep}, GraphCL~\cite{you2020graph}, GraphMAE~\cite{hou2022graphmae}, GraphACL~\cite{xiao2023simple}, S2GAE~\cite{tan2023s2gae};
(3) \textbf{GFMs:} GCOPE~\cite{sun2024gcope}, RiemannGFM~\cite{sun2025riemanngfm}, and GIT~\cite{wangtowards}. For fair comparison, all pre-training methods adopt the same two-layer GCN backbone unless otherwise specified.

\subsection{RQ1: Few-Shot Node and Graph Classification}
We evaluate SCGFM under a unified 5-shot, fixed-encoder protocol on both graph-level (Table~\ref{tab:few_shot_graph}) and node-level (Table~\ref{tab:few_shot_node}) transfer benchmarks.
Overall, SCGFM consistently outperforms strong Self-supervised Learning (SSL) baselines and recent GFM-style methods, showing superior robustness across domains and tasks.

\textbf{Graph Classification.}
SCGFM achieves strong in-domain performance on COX2, NCI1, and BZR, and maintains clear advantages in cross-domain transfer.
In particular, on COLORS-3—characterized by large structural shifts and an 11-way label space—most baselines degrade to near chance-level accuracy, while SCGFM remains substantially more robust.
This gain is not solely driven by COLORS-3: excluding its two transfer columns, SCGFM still attains a mean accuracy of $57.82\%$, outperforming the suboptimal RiemannGFM by $4.08\%$.

\textbf{Source-independent Stability in Node Classification.}
A striking observation from Table~\ref{tab:few_shot_node} is that SCGFM exhibits near source-independent behavior in node classification.
For example, on Reddit, performance remains consistently high ($\approx$$84.17\%$) regardless of the encoder’s pre-training domain.

To better understand this phenomenon, we analyze the similarity of target-domain embeddings produced by encoders pre-trained on different source domains (Figure~\ref{fig:tsne_comparison}).
We report \textbf{Centered Kernel Alignment (CKA $ \in [0,1]$)}~\cite{kornblith2019similarity} and the \textbf{Pearson correlation coefficient} ($\rho\in[-1,1]$) to quantify representation similarity across source domains, where larger values in both metrics indicate stronger agreement between the resulting embedding spaces.

Despite being trained on distinct sources, the resulting representations on Reddit display highly consistent latent geometries.
This suggests that SCGFM learns a dominant semantic structure in the embedding space that is largely invariant to the choice of pre-training domain.

\begin{figure}[!ht]
  \centering
  \includegraphics[width=0.9\columnwidth]{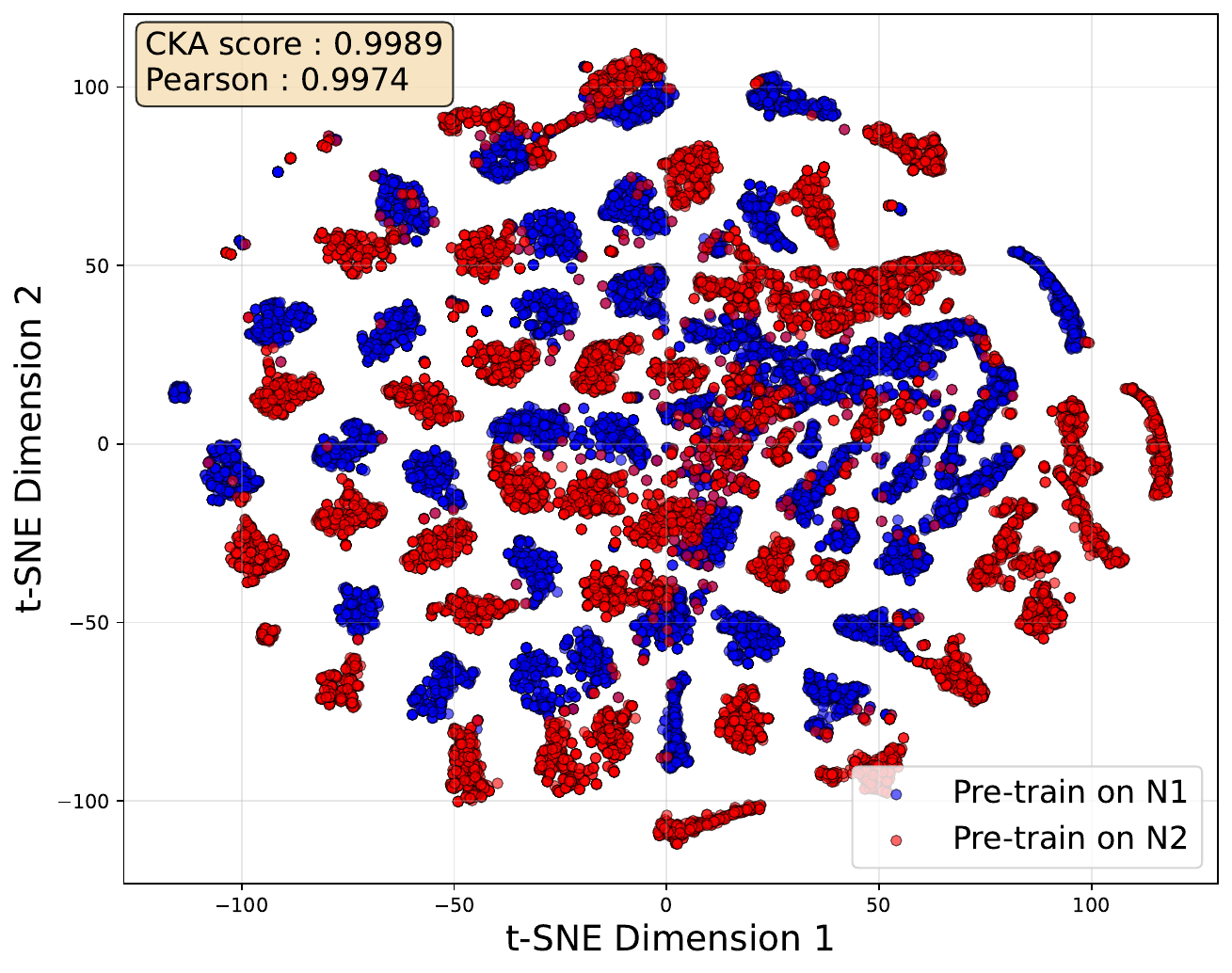}
  \caption[]{
      \textbf{Representation similarity across source domains on Reddit.}
      The t-SNE~\cite{maaten2008visualizing} visualization of node embeddings generated by encoders pre-trained on two different source domains.
  }
  \label{fig:tsne_comparison}
\end{figure}

This contrasts with the graph classification results in Table~\ref{tab:few_shot_graph}, where domain shifts lead to noticeable performance variations.
We attribute this difference to the dominance of high-dimensional semantics.
Unlike graph classification datasets (e.g., molecular graphs), which rely heavily on structural topology, node classification benchmarks are rich in semantic features.
SCGFM projects these features into an ultra–high-dimensional space ($\approx$45k dimensions on Cora with $M=32$).
According to \textbf{Cover's Theorem}~\cite{cover2006geometrical}, such a projection greatly increases the likelihood of linear separability, causing the decision boundary to be driven primarily by the intrinsic geometry of the semantic features rather than the pre-training domain.

Consequently, node classification performance is largely invariant to the pre-training source, and the remaining variance is dominated by task sampling effects rather than instability in the learned representations.
\begin{figure}[!t]
  \centering
  \includegraphics[width=0.7\linewidth]{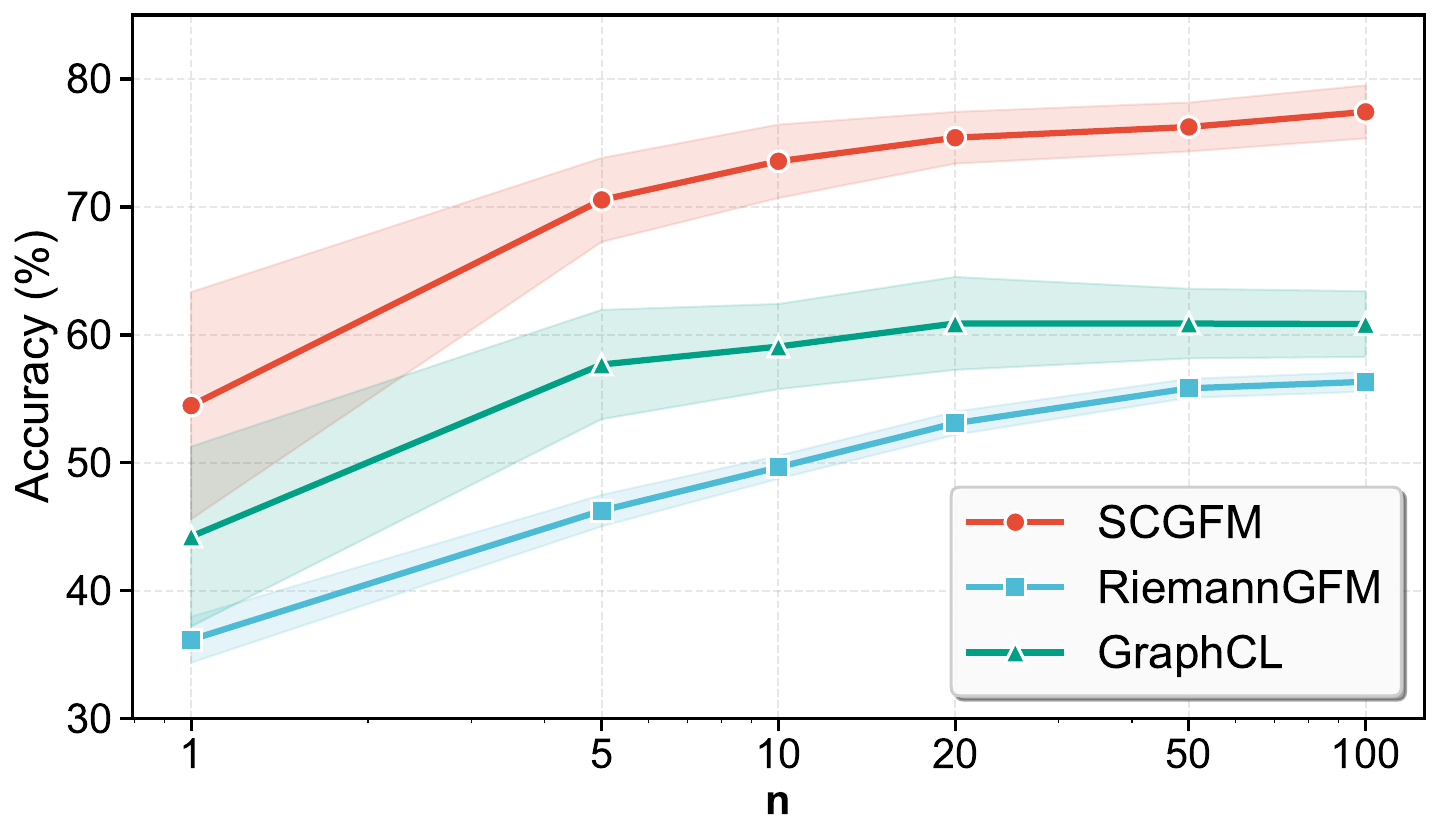}
  \caption{ \textbf{n-shot node classification on Cora.} The colored region denotes the variance of the results.
  }
  \label{fig:n_shot_node}
\end{figure}

\textbf{Impact of Shot Number.}
To assess robustness, we extend the evaluation across varying support sizes $n \in \{1, 5, \ldots, 100\}$.
As shown in Figure~\ref{fig:n_shot_node}, SCGFM consistently outperforms both baselines compared, with the largest margin observed in the extreme 1-shot regime, validating the effectiveness of our geometric inductive bias under severe data scarcity.
Moreover, performance improves steadily with increasing supervision, indicating that SCGFM refines class representations without overfitting to outlier-induced variance.

\begin{figure}[!t]
  \centering
  \includegraphics[width=\columnwidth]{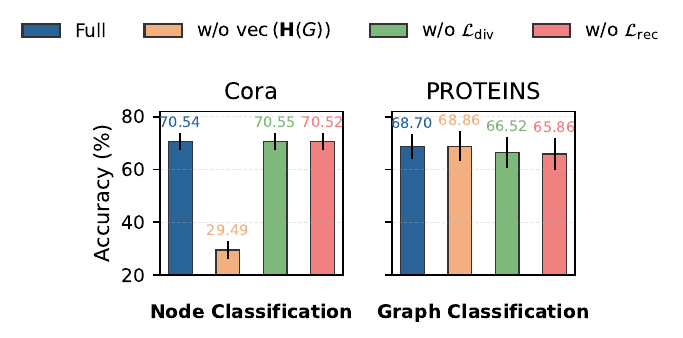}
  \caption[]{
      Ablation studies on Cora and PROTEINS (5-shot).
  }
  \label{fig:ablation_study}
\end{figure}
\subsection{RQ2: Ablation and Scalability}

\textbf{Ablation Analysis.} 
As illustrated in Figure~\ref{fig:ablation_study}, the contribution of individual components varies markedly across data modalities, revealing a mechanism of adaptive reliance.
On \textbf{Cora}, which features high-dimensional and informative node attributes, the learned representation is dominated by structure-aligned semantics $\mathrm{vec}(\mathbf{H}(G))$.
As expected, removing this component leads to a catastrophic performance drop ($70.54\% \rightarrow 29.49\%$), indicating that geometric coordinates alone are insufficient for content-heavy tasks.

In contrast, on the structure-centric \textbf{PROTEINS} dataset, removing $\mathrm{vec}(\mathbf{H}(G))$ yields a marginal performance improvement ($68.70\% \rightarrow 68.86\%$).We attribute this counterintuitive effect to semantic interference: when labels are primarily determined by intrinsic topology rather than node attributes, injecting weak or misaligned semantic signals ($\mathbf{H}(G)$) introduces noise that dilutes the discriminative power of the geometric coordinates $\mathbf{w}$.

Despite this modality-dependent dichotomy, the optimization constraints are universally critical.
Removing either the reconstruction loss $\mathcal{L}{\mathrm{rec}}$ or the diversity regularizer $\mathcal{L}{\mathrm{div}}$ consistently degrades performance across all datasets, confirming their necessity for learning a stable and non-degenerate geometric basis space.

\textbf{Scalability.}
We evaluate scalability on synthetic datasets consisting of 1k graphs, with the average graph size varying up to 5k nodes.
As shown in Figure~\ref{fig:scalability_comparison}, SCGFM scales smoothly to $5 \times 10^6$ total nodes while maintaining the lowest peak memory consumption ($\approx$12 GB) among all methods.
In contrast, \textbf{GAT} and \textbf{GIT} encounter out-of-memory (OOM) errors at $3 \times 10^6$ nodes due to their high computational overhead.
At $5 \times 10^6$ nodes, both \textbf{GIN} and \textbf{RiemannGFM} exceed memory limits because of their substantially higher memory usage.
\begin{figure}[t]
  \centering
  \includegraphics[width=\linewidth]{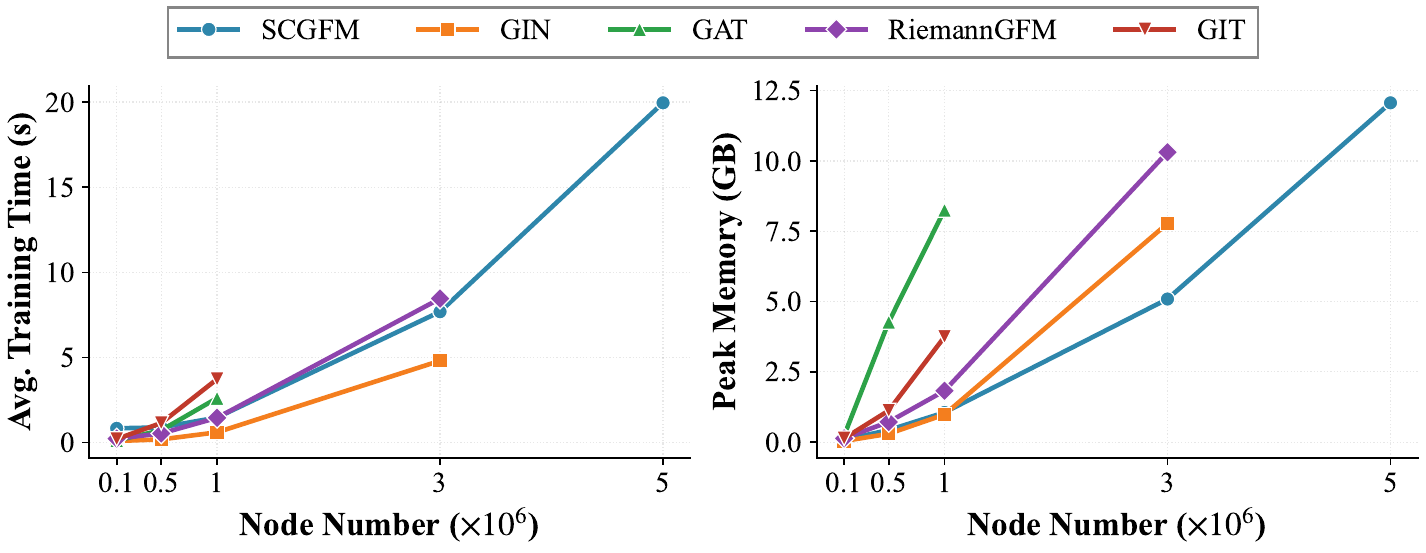}
  \caption[]{
      Scalability comparison on synthetic datasets.
  }
  \label{fig:scalability_comparison}
\end{figure}
\begin{figure}[t]
\vskip 0.1in
  \centering
  \includegraphics[width=\linewidth]{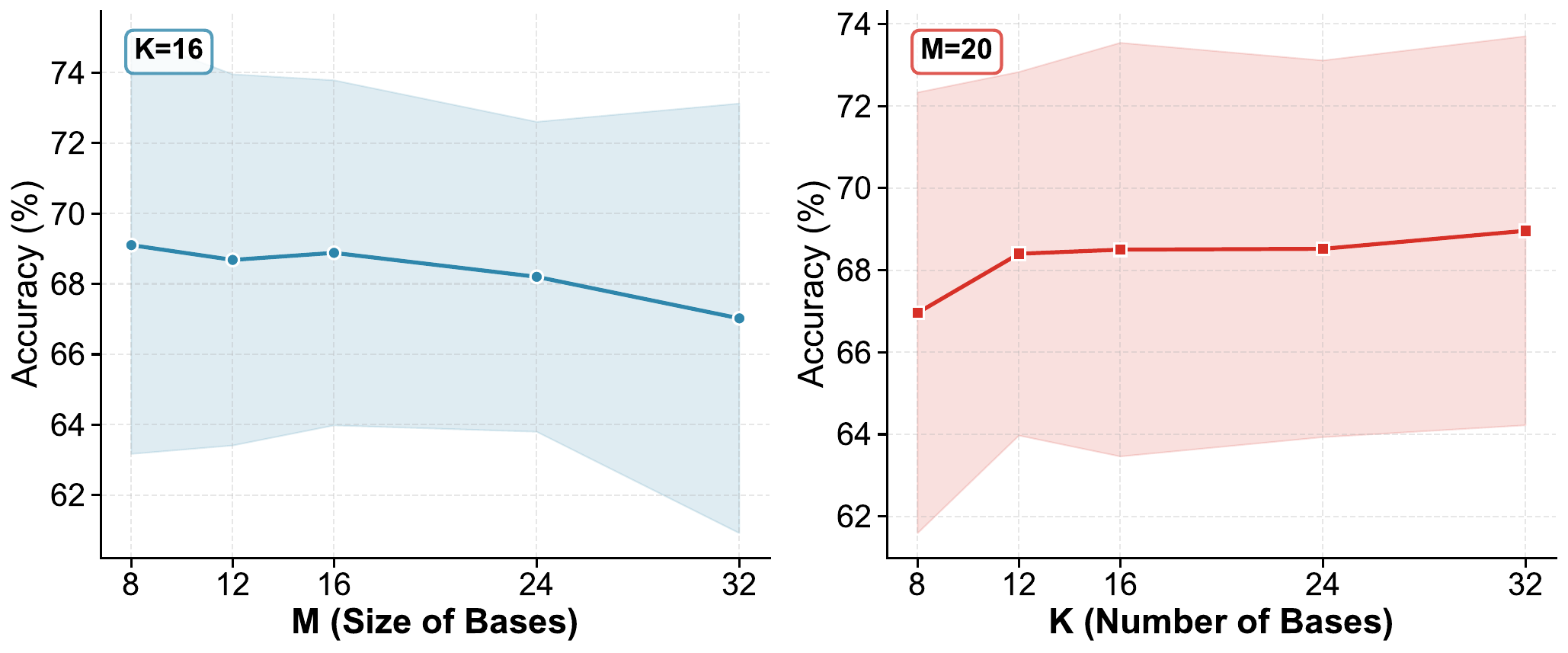}
  \caption{ Hyperparameter analysis on PROTEINS (5-shot).
  }
  \label{fig:hyperparameter}
\end{figure}

\begin{figure}[t]
  \centering
  \includegraphics[width=\columnwidth]{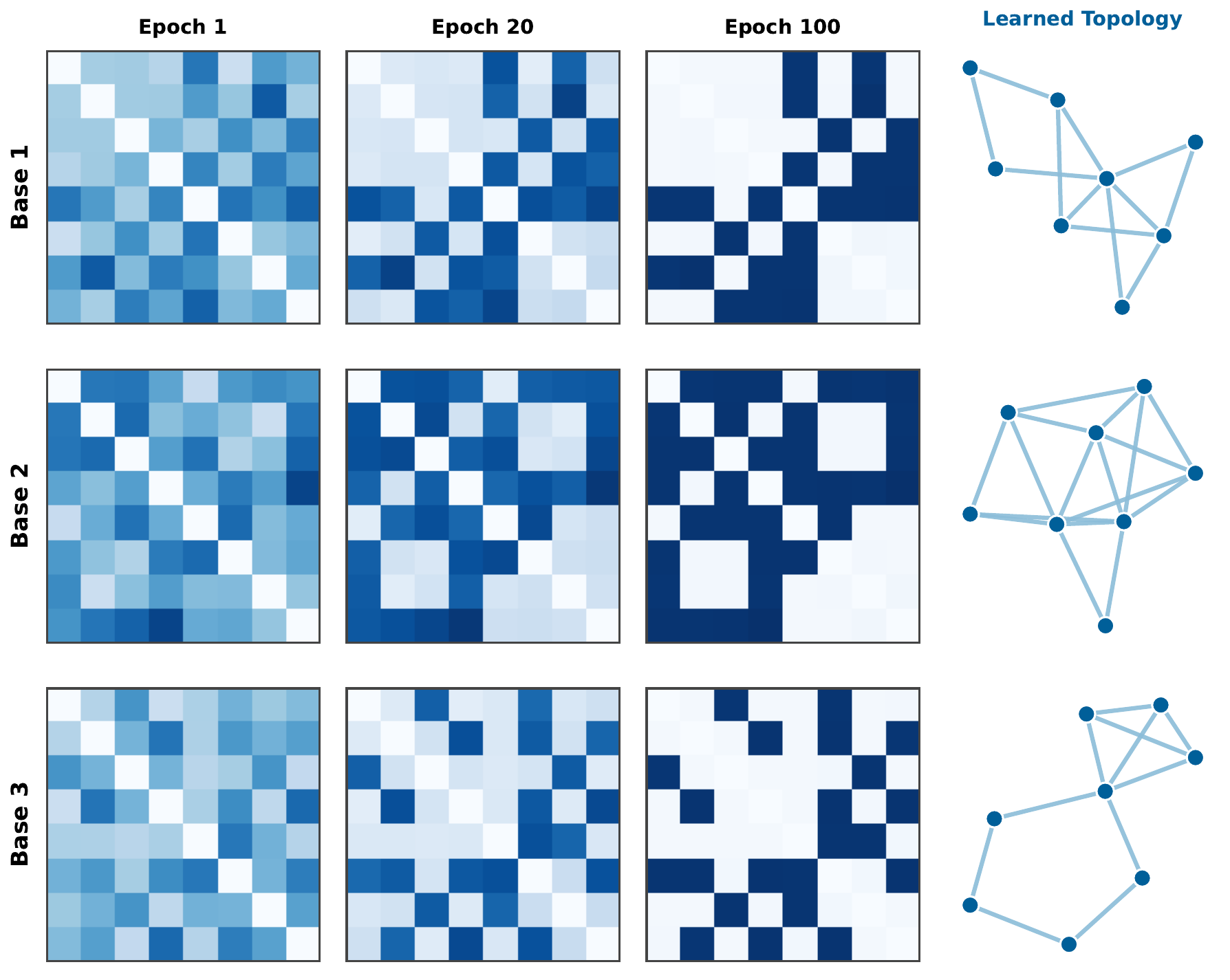}
  \caption[]{
    \textbf{Evolution of learned geometric bases on MUTAG.}
    Rows represent three distinct learned bases; columns display the pseudo-metric heatmaps during training (Epoch 1, 20, 100) and the final induced topology.
  }
    \label{fig:base_topology}
\end{figure}

\textbf{Pre-training Computational Analysis.}
The computational bottleneck of SCGFM lies in the SGW-based metric alignment
during pre-training.
By replacing exact graph matching with sliced projections, the complexity scales
as $O(KL(N\log N + M\log M))$, which reduces to $O(N\log N)$ in practice since the
number of bases $K$, the number of slices $L$, and base size $M$ are small
constants.
All remaining components introduce negligible overhead, while feature statistics
$\text{FE}(G)$ are computed once as an offline preprocessing step.

\textbf{Hyperparameters Sensitivity Analysis.}
We analyze the impact of the hyperparameters $K$ and $M$ on the performance of SCGFM (Figure~\ref{fig:hyperparameter}).
Increasing the number of bases $K$ improves accuracy up to $K=12$, beyond which performance plateaus around $68.5\% \to 69.0\%$. This suggests that a compact dictionary of 12–16 bases is sufficient to capture the underlying graph manifold in PROTEINS. SCGFM is also robust to the choice of base size $M$, exhibiting stable performance for $M \in [8,16]$, which aligns with the typical scale of functional motifs in protein graphs. Larger bases ($M=32$) introduce unnecessary complexity and slightly degrade performance.

\subsection{RQ3: Interpretability of Geometric Bases}
To evaluate the interpretability of the learned geometric bases, we visualize the training dynamics of three bases on the MUTAG dataset (Figure~\ref{fig:base_topology}).
The pseudo-metric heatmaps evolve from diffuse, unstructured patterns at initialization to sharp, high-contrast block structures after convergence, indicating that the model progressively suppresses noise and captures consistent structural relations.

The reconstructed graphs reveal that different bases specialize in distinct geometric motifs. Specifically, \textbf{Base 2} encodes a dense, highly clustered topology resembling aromatic-like substructures, whereas \textbf{Base 3} disentangles a cyclic core with a branching path, forming a recognizable ring-and-tail pattern.
Overall, the learned bases display clear specialization and diversity, demonstrating that SCGFM discovers meaningful geometric primitives rather than collapsing to a single trivial topology.

\begin{figure}[!t]
  \centering
  \includegraphics[width=0.9\linewidth]{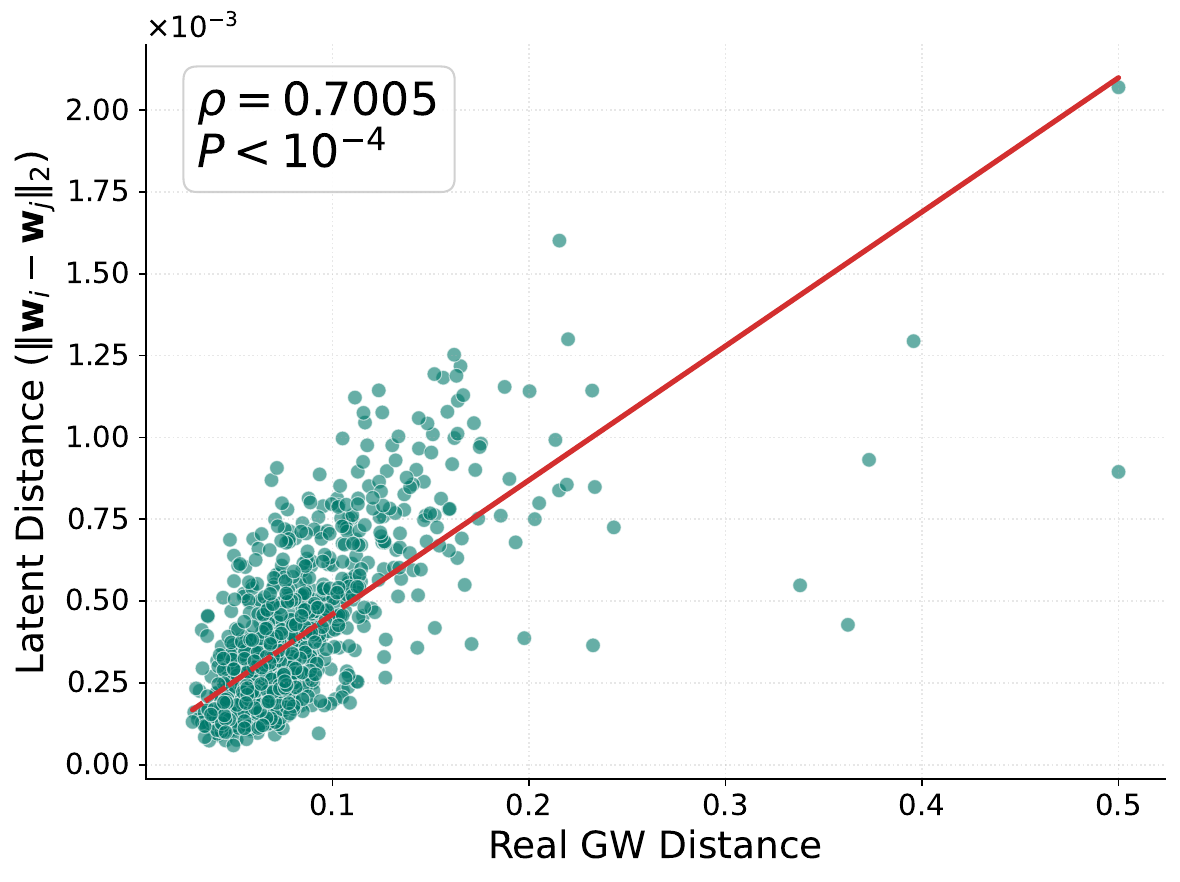}
  \caption{\textbf{ Approximate isometry on NCI1.}
  Scatter plot of latent Euclidean distances versus true GW distances. 
  The solid red line indicates the linear regression fit to the data.  
  Pearson correlation ($\rho$) and P-value are reported in the figure.
  }
  \label{fig:isometry_study}
\end{figure}

\subsection{RQ4: Geometry Alignment and Approximate Isometry}

We investigate whether the learned geometric bases induce an approximately isometric embedding, effectively mapping the complex mm-space of graphs into a structurally meaningful Euclidean vector space.
Specifically, our aim is to verify whether the Euclidean distances in the latent space consistently align with the intrinsic GW distances between graphs.
As shown in Figure~\ref{fig:isometry_study}, the latent distances learned by SCGFM exhibit a strong and statistically significant linear correlation with the true GW distances on NCI1 ($\rho = 0.7005$, $P < 10^{-4}$).
This quantitative result indicates that the latent space faithfully preserves the intrinsic geometry of graphs; that is, graph pairs with high topological similarity (low GW distance) are mapped to proximal points in the embedding space, while dissimilar graphs represent distant points.
This finding confirms that SCGFM produces an approximately isometric embedding, where Euclidean distances accurately reflect intrinsic graph similarities.

\section{Conclusion}
We introduce SCGFM, a structure-centric graph foundation model designed to overcome structural heterogeneity and feature mismatching in cross-domain graph learning.
By learning a set of geometric structural bases and leveraging Gromov-Wasserstein distances for projection, SCGFM projects heterogeneous graphs into a shared geometric space that cleanly decouples and aligns topology with semantics.
Extensive experiments show that SCGFM consistently surpasses state-of-the-art self-supervised methods and existing GFMs on few-shot node-and graph-level classification tasks across domains.
Beyond empirical gains, SCGFM demonstrates strong cross-domain generalization, robustness, and scalability, while providing theoretical guarantees on the stability of the structural coordinates and the reconstructed representations.

\section*{Impact Statement}
This paper presents work whose goal is to advance the field of Machine
Learning. There are many potential societal consequences of our work, none
which we feel must be specifically highlighted here.


\section*{Acknowledgements}
This work was supported by the National Natural Science Foundation of China (No. 62276053).

\bibliography{example_paper}

@article{memoli2011gromov,
  title={Gromov--Wasserstein distances and the metric approach to object matching},
  author={M{\'e}moli, Facundo},  
  journal={Foundations of Computational Mathematics},
  volume={11},
  number={4},
  pages={417--487},
  year={2011},
  publisher={Springer}
}

@article{peyre2016gromov,
  title={Gromov-Wasserstein averaging of kernel and distance matrices},
  author={Peyr{\'e}, Gabriel and Cuturi, Marco and Solomon, Justin},
  journal={International Conference on Machine Learning (ICML 2016)},
  series       = {{JMLR} Workshop and Conference Proceedings},
  volume       = {48},
  pages={2664--2672},
  year={2016},
  organization={PMLR}
}

@inproceedings{vayer2019sliced,
  title={Sliced {G}romov-{W}asserstein},  
  author={Vayer, Titouan and Flamary, R{\'e}mi and Tavenard, Romain and Chapel, Laetitia and Courty, Nicolas},
  booktitle={33rd Conference on Neural Information Processing Systems (NeurIPS 2019)},  
  year={2019},
  address={Vancouver, Canada},
  pages={10245--10255},
}

@book{Gromov1999Metric,
  title={Metric Structures for {R}iemannian and Non-{R}iemannian Spaces},
  author={{G}romov, {M}ikhail},
  translator={Bates, Sean Michael},
  publisher={Birkhäuser Boston},
  year={1999},
}

@inproceedings{velickovic2018deep,
  author       = {Petar Velickovic and
                  William Fedus and
                  William L. Hamilton and
                  Pietro Li{\`{o}} and
                  Yoshua Bengio and
                  R. Devon Hjelm},
  title        = {Deep Graph Infomax},
  booktitle    = {7th International Conference on Learning Representations (ICLR 2019)},
  address      = {New Orleans, LA, USA},
  publisher    = {OpenReview.net},
  year         = {2019},
}

@article{you2020graph,
title        = {Graph contrastive learning with augmentations},
author       = {You, Yuning and Chen, Tianlong and Sui, Yongduo and Chen, Ting and Wang, Zhangyang and Shen, Yang},
journal      = {Advances in Neural Information Processing Systems (NeurIPS 2020)},
volume       = {33},
pages        = {5812--5823},
year         = {2020},
}

@article{xiao2023simple,
title        = {Simple and asymmetric graph contrastive learning without augmentations},
author       = {Xiao, Teng and Zhu, Huaisheng and Chen, Zhengyu and Wang, Suhang},
journal      = {Advances in Neural Information Processing Systems (NeurIPS 2023)},
volume       = {36},
pages        = {16129--16152},
year         = {2023}
}

@inproceedings{hou2022graphmae,
title        = {Graph{MAE}: Self-Supervised Masked Graph Autoencoders},
author       = {Hou, Zhenyu and Liu, Xiao and Cen, Yukuo and Dong, Yuxiao and Yang, Hongxia and Wang, Chunjie and Tang, Jie},
booktitle    = {Proceedings of the 28th ACM SIGKDD Conference on Knowledge Discovery and Data Mining (KDD 2022)},
pages        = {594--604},
year         = {2022},
publisher    = {ACM},
eprint       = {2205.10803},
archivePrefix= {arXiv},
primaryClass = {cs.LG}
}

@inproceedings{hou2023graphmae2,
title        = {Graph{MAE2}: A decoding-enhanced masked self-supervised graph learner},
author       = {Hou, Zhenyu and He, Yufei and Cen, Yukuo and Liu, Xiao and Dong, Yuxiao and Kharlamov, Evgeny and Tang, Jie},
booktitle    = {Proceedings of the ACM Web Conference 2023 (WWW 2023)},
pages        = {737--746},
year         = {2023},
publisher    = {ACM},
}

@inproceedings{tan2023s2gae,
title        = {{S2GAE}: Self-supervised graph autoencoders are generalizable learners with graph masking},
author       = {Tan, Qiaoyu and Liu, Ninghao and Huang, Xiao and Choi, Soo-Hyun and Li, Li and Chen, Rui and Hu, Xia},
booktitle    = {Proceedings of the Sixteenth ACM International Conference on Web Search and Data Mining (WSDM 2023)},
pages        = {787--795},
year         = {2023},
publisher    = {ACM},
}

@inproceedings{liuone,
title        = {One For All: Towards Training One Graph Model For All Classification Tasks},
author       = {Liu, Hao and Feng, Jiarui and Kong, Lecheng and Liang, Ningyue and Tao, Dacheng and Chen, Yixin and Zhang, Muhan},
booktitle    = {International Conference on Learning Representations (ICLR 2024)},
year         = {2024},
}

@inproceedings{wangtowards,
  title        = {Towards Graph Foundation Models: Learning Generalities Across Graphs via Task-Trees},
  author       = {Wang, Zehong and Zhang, Zheyuan and Ma, Tianyi and Chawla, Nitesh V and Zhang, Chuxu and Ye, Yanfang},
  booktitle    = {Proceedings of the 42nd International Conference on Machine Learning (ICML 2025)},
  pages        = {65518--65555},
  year         = {2025},
  series       = {Proceedings of Machine Learning Research},
  volume       = {267},
  publisher    = {PMLR},
}

@inproceedings{wangmulti,
title        = {Multi-Domain Graph Foundation Models: Robust Knowledge Transfer via Topology Alignment},
author       = {Wang, Shuo and Wang, Bokui and Shen, Zhixiang and Deng, Boyan and others},
booktitle    = {Proceedings of the 42nd International Conference on Machine Learning (ICML 2025)},
pages        = {64806--64821},
year         = {2025},
series       = {Proceedings of Machine Learning Research},
volume       = {267},
publisher    = {PMLR},
}

@inproceedings{sun2025riemanngfm,
title        = {Riemann{GFM}: Learning a graph foundation model from riemannian geometry},
author       = {Sun, Li and Huang, Zhenhao and Zhou, Suyang and Wan, Qiqi and Peng, Hao and Yu, Philip},
booktitle    = {Proceedings of the ACM Web Conference 2025 (WWW 2025)},
pages        = {1154--1165},
year         = {2025},
publisher    = {ACM},
eprint       = {2502.03251},
archivePrefix= {arXiv},
primaryClass = {cs.LG}
}

@article{liu2023graphssl,
  title   = {Graph Self-Supervised Learning: A Survey},
  author  = {Liu, Yixin and Jin, Ming and Pan, Shirui and Zhou, Chuan and Zheng, Yu and Xia, Feng and Yu, Philip S.},
  journal = {IEEE Transactions on Knowledge and Data Engineering (TKDE 2023)},
  volume  = {35},
  number  = {6},
  pages   = {5879--5900},
  year    = {2023},
}

@article{xie2023sslgnn,
  title   = {Self-Supervised Learning of Graph Neural Networks: A Unified Review},
  author  = {Xie, Yaochen and Xu, Zhao and Zhang, Jingtun and Wang, Zhengyang and Ji, Shuiwang},
  journal = {IEEE Transactions on Pattern Analysis and Machine Intelligence (TPAMI 2023)},
  volume  = {45},
  number  = {2},
  pages   = {2412--2429},
  year    = {2023},
}

@article{liu2025gfm,
  title   = {Graph Foundation Models: Concepts, Opportunities and Challenges},
  author  = {Liu, Jiawei and Yang, Cheng and Lu, Zhiyuan and Chen, Junze and Li, Yibo and Zhang, Mengmei and Bai, Ting and Fang, Yuan and Sun, Lichao and Yu, Philip S. and Shi, Chuan},
  journal = {IEEE Transactions on Pattern Analysis and Machine Intelligence (TPAMI 2025)},
  volume  = {47},
  number  = {6},
  pages   = {5023--5044},
  year    = {2025},
}

@article{wang2025graph,
  title={Graph Foundation Models: A Comprehensive Survey},
  author={Wang, Zehong and Liu, Zheyuan and Ma, Tianyi and Li, Jiazheng and Zhang, Zheyuan and Fu, Xingbo and Li, Yiyang and Yuan, Zhengqing and Song, Wei and Ma, Yijun and others},
  journal={arXiv preprint arXiv:2505.15116},
  volume       = {abs/2502.08346},
  pages        = {6184--6194},
  year={2025}
}

@inproceedings{devlin-etal-2019-bert,
    title = "{BERT}: Pre-training of Deep Bidirectional Transformers for Language Understanding",
    author = "{D}evlin, {J}acob  and
      {C}hang, {M}ing-{W}ei  and
      {L}ee, {K}enton  and
      {T}outanova, {K}ristina",
    booktitle = "Proceedings of the 2019 Conference of the North {A}merican Chapter of the Association for Computational Linguistics: Human Language Technologies, Volume 1 (NAACL-HLT 2019)",
    year = "2019",
    pages = "4171--4186",
    address = "Minneapolis, Minnesota",
    publisher = "Association for Computational Linguistics",
}

@inproceedings{NEURIPS2020_1457c0d6,
 author = {{B}rown, {T}om and Mann, Benjamin and Ryder, Nick and Subbiah, Melanie and Kaplan, Jared D and Dhariwal, Prafulla and Neelakantan, Arvind and Shyam, Pranav and Sastry, Girish and Askell, Amanda and Agarwal, Sandhini and Herbert-Voss, Ariel and Krueger, Gretchen and Henighan, Tom and Child, Rewon and Ramesh, Aditya and Ziegler, Daniel and Wu, Jeffrey and Winter, Clemens and Hesse, Chris and Chen, Mark and Sigler, Eric and Litwin, Mateusz and Gray, Scott and Chess, Benjamin and Clark, Jack and Berner, Christopher and McCandlish, Sam and Radford, Alec and Sutskever, Ilya and Amodei, Dario},
 booktitle = {Advances in Neural Information Processing Systems (NeurIPS 2020)},
 pages = {1877--1901},
 publisher = {Curran Associates, Inc.},
 title = {Language Models are Few-Shot Learners},
 volume = {33},
 year = {2020}
}

@inproceedings{dosovitskiy2021image,
  title={An Image is Worth 16x16 Words: Transformers for Image Recognition at Scale},
  author={Dosovitskiy, Alexey and Beyer, Lucas and Kolesnikov, Alexander and Weissenborn, Dirk and Zhai, Xiaohua and Unterthiner, Thomas and Dehghani, Mostafa and Minderer, Matthias and Heigold, Georg and Gelly, Sylvain and others},
  booktitle={International Conference on Learning Representations (ICLR 2021)},
  year={2021}
}

@InProceedings{pmlr-v267-chen25cf,
  title = 	 {Hierarchical Graph Tokenization for Molecule-Language Alignment},
  author =       {Chen, Yongqiang and Yao, Quanming and Zhang, Juzheng and Cheng, James and Bian, Yatao},
  booktitle = 	 {Proceedings of the 42nd International Conference on Machine Learning (ICML 2025)},
  pages = 	 {9664--9690},
  year = 	 {2025},
  volume = 	 {267},
  series = 	 {Proceedings of Machine Learning Research},
  month = 	 {13--19 Jul},
  publisher =    {PMLR},
}

@inproceedings{he2022mae,
  author    = {He, Kaiming and Chen, Xinlei and Xie, Saining and
               Li, Yanghao and Doll{\'a}r, Piotr and Girshick, Ross},
  title     = {Masked Autoencoders Are Scalable Vision Learners},
  booktitle = {Proceedings of the IEEE/CVF Conference on Computer Vision and Pattern Recognition (CVPR 2022)},
  year      = {2022}
}

@inproceedings{ye-etal-2024-language,
    title = "Language is All a Graph Needs",
    author = "Ye, Ruosong  and
      Zhang, Caiqi  and
      Wang, Runhui  and
      Xu, Shuyuan  and
      Zhang, Yongfeng",
    booktitle = "Findings of the Association for Computational Linguistics: EACL 2024",
    year = "2024",
    address = "St. Julian{'}s, Malta",
    publisher = "Association for Computational Linguistics",
    pages = "1955--1973",
}

@inproceedings{NEURIPS2023_622afc4e,
 author = {Wang, Heng and Feng, Shangbin and He, Tianxing and Tan, Zhaoxuan and Han, Xiaochuang and Tsvetkov, Yulia},
 booktitle = {Advances in Neural Information Processing Systems (NeurIPS 2023)},
 pages = {30840--30861},
 publisher = {Curran Associates, Inc.},
 title = {Can Language Models Solve Graph Problems in Natural Language?},
 volume = {36},
 year = {2023}
}

@inproceedings{zhao2023gimlet,
 author = {Zhao, Haiteng and Liu, Shengchao and Chang, Ma and Xu, Hannan and Fu, Jie and Deng, Zhihong and Kong, Lingpeng and Liu, Qi},
 booktitle = {Advances in Neural Information Processing Systems (NeurIPS 2023)},
 pages = {5850--5887},
 publisher = {Curran Associates, Inc.},
 title = {{GIMLET}: A Unified Graph-Text Model for Instruction-Based Molecule Zero-Shot Learning},
 volume = {36},
 year = {2023}
}

@inproceedings{yuan2025how,
  title={How Much Can Transfer? {BRIDGE}: Bounded Multi-Domain Graph Foundation Model with Generalization Guarantees},
  author={Haonan Yuan and Qingyun Sun and Junhua Shi and Xingcheng Fu and Bryan Hooi and Jianxin Li and Philip S. Yu},
  booktitle={Proceedings of the Forty-second International Conference on Machine Learning (ICML 2025)},
  year={2025},
}

@article{Bruna2013SpectralNA,
  title={Spectral Networks and Locally Connected Networks on Graphs},
  author={Joan Bruna and Wojciech Zaremba and Arthur Szlam and Yann LeCun},
  journal={CoRR},
  year={2013},
  volume={abs/1312.6203},
}

@inproceedings{Hu*2020Strategies,
title={Strategies for Pre-training Graph Neural Networks},
author={Weihua Hu and Bowen Liu and Joseph Gomes and Marinka Zitnik and Percy Liang and Vijay Pande and Jure Leskovec},
booktitle={8th International Conference on Learning Representations (ICLR 2020)},
year={2020},
address = {Addis Ababa, Ethiopia}
}

@inproceedings{NEURIPS2022_5d4834a1,
 author = {Ramp\'{a}\v{s}ek, Ladislav and Galkin, Michael and Dwivedi, Vijay Prakash and Luu, Anh Tuan and Wolf, Guy and Beaini, Dominique},
 booktitle = {Advances in Neural Information Processing Systems (NeurIPS 2022)},
 pages = {14501--14515},
 publisher = {Curran Associates, Inc.},
 title = {Recipe for a General, Powerful, Scalable Graph Transformer},
 volume = {35},
 year = {2022}
}

@inproceedings{10.5555/3618408.3620116,
author = {Zeng, Zhichen and Zhu, Ruike and Xia, Yinglong and Zeng, Hanqing and Tong, Hanghang},
title = {Generative graph dictionary learning},
year = {2023},
booktitle = {Proceedings of the 40th International Conference on Machine Learning (ICML 2023)},
pages={40749--40769},
organization={PMLR},
location = {Honolulu, Hawaii, USA},
}

@InProceedings{pmlr-v139-vincent-cuaz21a,
  title = 	 {Online Graph Dictionary Learning},
  author =       {Vincent-Cuaz, C{\'e}dric and Vayer, Titouan and Flamary, R{\'e}mi and Corneli, Marco and Courty, Nicolas},
  booktitle = 	 {Proceedings of the 38th International Conference on Machine Learning (ICML 2021)},
  pages = 	 {10564--10574},
  year = 	 {2021},
  volume = 	 {139},
  series = 	 {Proceedings of Machine Learning Research},
  month = 	 {18--24 Jul},
  publisher =    {PMLR},
}

@article{gao2017properties,
  title={On the properties of the softmax function with application in game theory and reinforcement learning},
  author={Gao, Bolin and Pavel, Lilianna},
  journal={arXiv preprint arXiv:1704.00805},
  year={2017}
}

@inproceedings{kipf2017gcn,
  title={Semi-Supervised Classification with Graph Convolutional Networks},
  author={Thomas N. Kipf and Max Welling},
  booktitle={5th International Conference on Learning Representations (ICLR 2017)},
  year={2017},
  address={Toulon, France},
}

@inproceedings{xu2018gin,
  title={How Powerful are Graph Neural Networks?},
  author={Keyulu Xu and Weihua Hu and Jure Leskovec and Stefanie Jegelka},
  booktitle={7th International Conference on Learning Representations (ICLR 2019)},
  address={New Orleans, LA, USA},
  year={2019},
}

@inproceedings{veličković2018gat,
title={Graph Attention Networks},
author={Petar Veličković and Guillem Cucurull and Arantxa Casanova and Adriana Romero and Pietro Liò and Yoshua Bengio},
booktitle={6th International Conference on Learning Representations (ICLR 2018)},
year={2018},
}

@inproceedings{sun2024gcope,
author = {Sun, Xiangguo and Cheng, Hong and Li, Jia and Liu, Bo and Guan, Jihong},
title = {All in one: multi-task prompting for graph neural networks},
year = {2024},
booktitle = {Proceedings of the Thirty-Third International Joint Conference on Artificial Intelligence (IJCAI 2024)},
pages        = {8460--8465},
address = {Jeju, South Korea},
publisher    = {ijcai.org},
location = {Jeju, Korea},
}

@InProceedings{pmlr-v267-wang25ez,
  title = 	 {Equivalence is All: A Unified View for Self-supervised Graph Learning},
  author =       {Wang, Yejiang and Zhao, Yuhai and Wang, Zhengkui and Li, Ling and Wang, Jiapu and Li, Fangting and Huang, Miaomiao and Pan, Shirui and Wang, Xingwei},
  booktitle = 	 {Proceedings of the 42nd International Conference on Machine Learning (ICML 2025)},
  pages = 	 {65776--65789},
  year = 	 {2025},
  address = {Vancouver, BC, Canada},
  month = 	 {13--19 Jul},
  publisher =    {OpenReview.net},
}

@InProceedings{pmlr-v267-li25ct,
  title = 	 {Self-supervised Masked Graph Autoencoder via Structure-aware Curriculum},
  author =       {Li, Haoyang and Wang, Xin and Zhang, Zeyang and Wu, Zongyuan and Xiao, Linxin and Zhu, Wenwu},
  booktitle = 	 {Proceedings of the 42nd International Conference on Machine Learning (ICML 2025)},
  pages = 	 {36215--36235},
  year = 	 {2025},
  volume = 	 {267},
  series = 	 {Proceedings of Machine Learning Research},
  month = 	 {13--19 Jul},
  publisher =    {PMLR},
}

@inproceedings{yuan2026rag,
  title={RAG-GFM: Overcoming In-Memory Bottlenecks in Graph Foundation Models via Retrieval-Augmented Generation},
  author={Yuan, Haonan and Sun, Qingyun and Tao, Jiacheng and Fu, Xingcheng and Li, Jianxin},
  booktitle={Proceedings of the ACM Web Conference 2026},
  pages={626--637},
  year={2026}
}

@inproceedings{fangsaga,
  title={SAGA: Structural Aggregation Guided Alignment with Dynamic View and Neighborhood Order Selection for Multiview Graph Domain Adaptation},
  author={Fang, Ruiyi and Zhao, Jingyu and Wang, Shuo and Pu, Ruizhi and Li, Bingheng and Cai, Jiale and Li, Zhihao and Jing, Zihao and Zhu, Jian and Tang, Song and others},
  booktitle={The Fourteenth International Conference on Learning Representations},
  year={2025}
}

@inproceedings{fang2026graph,
  title={Graph domain adaptation via homophily-agnostic reconstructing structure},
  author={Fang, Ruiyi and Wang, Shuo and Pu, Ruizhi and Zeng, Qiuhao and Zheng, Hao and Wang, Ziyan and Cai, Jiale and Mei, Zhimin and Tang, Song and Ling, Charles and others},
  booktitle={Proceedings of the AAAI Conference on Artificial Intelligence},
  volume={40},
  number={25},
  pages={21047--21055},
  year={2026}
}

@inproceedings{wang2026subtle,
  title={From Subtle to Significant: Prompt-Driven Self-Improving Optimization in Test-Time Graph OOD Detection},
  author={Wang, Luzhi and Fu, Xuanshuo and Zhang, He and Liu, Chuang and Wang, Xiaobao and Liu, Hongbo},
  booktitle={Proceedings of the AAAI Conference on Artificial Intelligence},
  volume={40},
  number={19},
  pages={15851--15859},
  year={2026}
}

@article{Yu2025SAMGPTTG,
  title={SAMGPT: Text-free Graph Foundation Model for Multi-domain Pre-training and Cross-domain Adaptation},
  author={Xingtong Yu and Zechuan Gong and Chang Zhou and Yuan Fang and Hui Zhang},
  journal={Proceedings of the ACM on Web Conference 2025},
  year={2025},
}

@inproceedings{NEURIPS2025_0c3ce12a,
 author = {Wang, Xingliang and Liu, Zemin and Han, Junxiao and Deng, Shuiguang},
 booktitle = {Advances in Neural Information Processing Systems},
 editor = {D. Belgrave and C. Zhang and H. Lin and R. Pascanu and P. Koniusz and M. Ghassemi and N. Chen},
 pages = {8252--8276},
 publisher = {Curran Associates, Inc.},
 title = {RAG4GFM: Bridging Knowledge Gaps in Graph Foundation Models through Graph Retrieval Augmented Generation},
 volume = {38},
 year = {2025}
}

@InProceedings{pmlr-v97-titouan19a,
  title = 	 {Optimal Transport for structured data with application on graphs},
  author =       {Titouan, Vayer and Courty, Nicolas and Tavenard, Romain and Laetitia, Chapel and Flamary, R{\'e}mi},
  booktitle = 	 {Proceedings of the 36th International Conference on Machine Learning},
  pages = 	 {6275--6284},
  year = 	 {2019},
  volume = 	 {97},
  series = 	 {Proceedings of Machine Learning Research},
  month = 	 {09--15 Jun},
  publisher =    {PMLR},
}

@InProceedings{pmlr-v202-chen23ak,
  title = 	 {A Gromov-{W}asserstein Geometric View of Spectrum-Preserving Graph Coarsening},
  author =       {Chen, Yifan and Yao, Rentian and Yang, Yun and Chen, Jie},
  booktitle = 	 {Proceedings of the 40th International Conference on Machine Learning},
  pages = 	 {5257--5281},
  year = 	 {2023},
  volume = 	 {202},
  series = 	 {Proceedings of Machine Learning Research},
  month = 	 {23--29 Jul},
  publisher =    {PMLR},
}

@InProceedings{pmlr-v202-pan23b,
  title = 	 {Beyond Homophily: Reconstructing Structure for Graph-agnostic Clustering},
  author =       {Pan, Erlin and Kang, Zhao},
  booktitle = 	 {Proceedings of the 40th International Conference on Machine Learning},
  pages = 	 {26868--26877},
  year = 	 {2023},
  volume = 	 {202},
  series = 	 {Proceedings of Machine Learning Research},
  month = 	 {23--29 Jul},
  publisher =    {PMLR},
}

@InProceedings{pmlr-v267-wang25an,
  title = 	 {Cooperation of Experts: Fusing Heterogeneous Information with Large Margin},
  author =       {Wang, Shuo and Huang, Shunyang and Yuan, Jinghui and Shen, Zhixiang and Kang, Zhao},
  booktitle = 	 {Proceedings of the 42nd International Conference on Machine Learning},
  pages = 	 {63169--63185},
  year = 	 {2025},
  volume = 	 {267},
  series = 	 {Proceedings of Machine Learning Research},
  month = 	 {13--19 Jul},
  publisher =    {PMLR},
}

@inproceedings{NEURIPS2024_380a0b16,
 author = {Shen, Zhixiang and Wang, Shuo and Kang, Zhao},
 booktitle = {Advances in Neural Information Processing Systems},
 doi = {10.52202/079017-0993},
 editor = {A. Globerson and L. Mackey and D. Belgrave and A. Fan and U. Paquet and J. Tomczak and C. Zhang},
 pages = {31629--31658},
 publisher = {Curran Associates, Inc.},
 title = {Beyond Redundancy: Information-aware Unsupervised Multiplex Graph Structure Learning},
 volume = {37},
 year = {2024}
}

@ARTICLE{shenwhen,
  author={Shen, Zhixiang and Kang, Zhao},
  journal={IEEE Transactions on Neural Networks and Learning Systems}, 
  title={When Heterophily Meets Heterogeneous Graphs: Latent Graphs Guided Unsupervised Representation Learning}, 
  year={2025},
  volume={36},
  number={6},
  pages={10283-10296},
  }

@inproceedings{fang2025benefits,
  title={On the benefits of attribute-driven graph domain adaptation},
  author={Fang, Ruiyi and Li, Bingheng and Zeng, Qiuhao and Hosseini Dashtbayaz, Nima and Pu, Ruizhi and Ling, Charles and Wang, Boyu and others},
  booktitle={International conference on learning representations},
  volume={2025},
  pages={30762--30782},
  year={2025}
}

@inproceedings{li2024pc,
  title={Pc-conv: Unifying homophily and heterophily with two-fold filtering},
  author={Li, Bingheng and Pan, Erlin and Kang, Zhao},
  booktitle={Proceedings of the AAAI conference on artificial intelligence},
  volume={38},
  number={12},
  pages={13437--13445},
  year={2024}
}

@inproceedings{10.1145/3690624.3709277,
author = {He, Yufei and Sui, Yuan and He, Xiaoxin and Hooi, Bryan},
title = {Uni{G}raph: Learning a Unified Cross-Domain Foundation Model for Text-Attributed Graphs},
year = {2025},
publisher = {Association for Computing Machinery},
address = {New York, NY, USA},
booktitle = {Proceedings of the 31st ACM SIGKDD Conference on Knowledge Discovery and Data Mining V.1 (KDD 2025)},
pages = {448–459},
numpages = {12},
location = {Toronto ON, Canada},
series = {KDD 2025}
}

@inproceedings{Zeng2020GraphSAINT,
  title={Graph{SAINT}: Graph Sampling Based Inductive Learning Method},
  author={Hanqing Zeng and Hongkuan Zhou and Ajitesh Srivastava and Rajgopal Kannan and Viktor Prasanna},
  booktitle={8th International Conference on Learning Representations (ICLR 2020)},
  address = {Addis Ababa, Ethiopia},
  year={2020},
}

@article{cover2006geometrical,
  author       = {Thomas M. Cover},
  title        = {Geometrical and Statistical Properties of Systems of Linear Inequalities with Applications in Pattern Recognition},
  journal      = {{IEEE} Trans. Electron. Comput.},
  volume       = {14},
  number       = {3},
  pages        = {326--334},
  year         = {1965},
}

@article{sutherland2003spline,
  title={Spline-fitting with a genetic algorithm: a method for developing classification structure-activity relationships},
  author={Sutherland, J. J. and O'Brien, L. A. and Weaver, D. F.},
  journal={Journal of Chemical Information and Computer Sciences},
  volume={43},
  number={6},
  pages={1906--1915},
  year={2003},
  publisher={ACS Publications},
  pmid={14632439}
}

@article{wale2008comparison,
  title={Comparison of descriptor spaces for chemical compound retrieval and classification},
  author={Wale, N. and Watson, I. A. and Karypis, G.},
  journal={Knowledge and Information Systems},
  volume={14},
  pages={347--375},
  year={2008},
}

@article{debNath1991structure,
  title={Structure-activity relationship of mutagenic aromatic and heteroaromatic nitro compounds. Correlation with molecular orbital energies and hydrophobicity},
  author={Debnath, A. K. and Lopez de Compadre, R. L. and Debnath, G. and Shusterman, A. J. and Hansch, C.},
  journal={Journal of Medicinal Chemistry},
  volume={34},
  number={2},
  pages={786--97},
  year={1991},
  publisher={ACS Publications},
  pmid={1995902}
}

@inproceedings{10.1145/2783258.2783417,
  author = {Yanardag, Pinar and Vishwanathan, S.V.N.},
  title = {Deep Graph Kernels},
  year = {2015},
  publisher = {Association for Computing Machinery},
  address = {New York, NY, USA},
  booktitle = {Proceedings of the 21th ACM SIGKDD International Conference on Knowledge Discovery and Data Mining (KDD 2015)},
  pages = {1365-1374},
  numpages = {10},
  location = {Sydney, NSW, Australia},
  series = {KDD 2015}
}

@inbook{10.5555/3454287.3454665,
  author = {Knyazev, Boris and Taylor, Graham W. and Amer, Mohamed R.},
  title = {Understanding attention and generalization in graph neural networks},
  year = {2019},
  publisher = {Curran Associates Inc.},
  address = {Vancouver, BC, Canada},
  booktitle = {Proceedings of the 33rd International Conference on Neural Information Processing Systems (NeurIPS 2019)},
  pages        = {4204--4214},
}

@article{10.1093/bioinformatics/bti1007,
  author = {Borgwardt, Karsten M. and Ong, Cheng Soon and Sch\"{o}nauer, Stefan and Vishwanathan, S. V. N. and Smola, Alex J. and Kriegel, Hans-Peter},
  title = {Protein function prediction via graph kernels},
  year = {2005},
  issue_date = {January 2005},
  publisher = {Oxford University Press, Inc.},
  address = {USA},
  volume = {21},
  number = {1},
  journal = {Bioinformatics},
  month = jan,
  pages = {47-56},
}

@InProceedings{pmlr-v48-yanga16,
  title = 	 {Revisiting Semi-Supervised Learning with Graph Embeddings},
  author = 	 {Yang, Zhilin and Cohen, William and Salakhudinov, Ruslan},
  booktitle = 	 {Proceedings of The 33rd International Conference on Machine Learning (ICML 2016)},
  pages = 	 {40--48},
  year = 	 {2016},
  volume = 	 {48},
  series = 	 {Proceedings of Machine Learning Research},
  address = 	 {New York, New York, USA},
  publisher =    {PMLR},
}

@article{Shchur2018PitfallsOG,
  title={Pitfalls of Graph Neural Network Evaluation},
  author={Oleksandr Shchur and Maximilian Mumme and Aleksandar Bojchevski and Stephan G{\"u}nnemann},
  journal={ArXiv},
  year={2018},
  volume={abs/1811.05868},
}

@inproceedings{10.5555/3294771.3294869,
  author = {Hamilton, William L. and Ying, Rex and Leskovec, Jure},
  title = {Inductive representation learning on large graphs},
  year = {2017},
  publisher = {Curran Associates Inc.},
  address = {Red Hook, NY, USA},
  booktitle = {Proceedings of the 31st International Conference on Neural Information Processing Systems (NeurIPS 2017)},
  pages = {1024-1034},
  numpages = {11},
  location = {Long Beach, California, USA},
  series = {NIPS 2017}
}

@article{Bronstein2021GeometricDL,
  title={Geometric Deep Learning: Grids, Groups, Graphs, Geodesics, and Gauges},
  author={Michael M. Bronstein and Joan Bruna and Taco Cohen and Petar Velivckovi'c},
  journal={ArXiv},
  year={2021},
  volume={abs/2104.13478},
}

@inproceedings{pmlr-v80-you18a,
  author       = {Jiaxuan You and
                  Rex Ying and
                  Xiang Ren and
                  William L. Hamilton and
                  Jure Leskovec},
  title        = {Graph{RNN}: Generating Realistic Graphs with Deep Auto-regressive Models},
  booktitle    = {Proceedings of the 35th International Conference on Machine Learning (ICML 2018)},
  series       = {Proceedings of Machine Learning Research},
  volume       = {80},
  pages        = {5694--5703},
  publisher    = {{PMLR}},
  year         = {2018},
  address      = {Stockholmsm{\"{a}}ssan, Stockholm, Sweden},
}

@inproceedings{pmlr-v80-jin18a,
  author       = {Wengong Jin and
                  Regina Barzilay and
                  Tommi S. Jaakkola},
  title        = {Junction Tree Variational Autoencoder for Molecular Graph Generation},
  booktitle    = {Proceedings of the 35th International Conference on Machine Learning (ICML 2018)},
  series       = {Proceedings of Machine Learning Research},
  address      = {Stockholmsm{\"{a}}ssan, Stockholm, Sweden},
  volume       = {80},
  pages        = {2328--2337},
  publisher    = {{PMLR}},
  year         = {2018},
}

@inproceedings{kornblith2019similarity,
  title     = {Similarity of Neural Network Representations Revisited},
  author    = {Kornblith, Simon and Norouzi, Mohammad and Lee, Honglak and Hinton, Geoffrey E.},
  booktitle = {Proceedings of the 36th International Conference on Machine Learning (ICML 2019)},
  pages={3519--3529},
  organization={PMlR},
  year      = {2019}
}

@inproceedings{ICLR2025_f2059277,
 author = {Wang, Limei and Hassani, Kaveh and Zhang, Si and Fu, Dongqi and Yuan, Baichuan and Cong, Weilin and Hua, Zhigang and Wu, Hao and Yao, Ning and Long, Bo},
 booktitle = {International Conference on Learning Representations (ICLR 2025)},
 pages = {97239--97260},
 title = {Learning Graph Quantized Tokenizers},
 volume = {2025},
 year = {2025}
}

@article{maaten2008visualizing,
  title={Visualizing data using t-SNE},
  author={Maaten, Laurens van der and Hinton, Geoffrey},
  journal={Journal of machine learning research},
  volume={9},
  number={Nov},
  pages={2579--2605},
  year={2008}
}
\bibliographystyle{icml2026}

\newpage
\appendix
\onecolumn
\section{Notations}
\label{sec:appendix_notion}
\begin{table}[h]
  \caption{Summary of Notations.}
  \begin{center}
    \begin{small}
      \begin{tabular}{ll}
        \toprule
        \textbf{Symbol} & \textbf{Description} \\
        \midrule
        $G = (\mathcal{V}, \mathcal{E},\mathbf{A},\mathbf{X})$ & An undirected graph with node set $\mathcal{V}$, edge set $\mathcal{E}$. \\
                                                                &Adjacency matrix $\mathbf{A}$ and feature matrix $\mathbf{X}$.\\
        $|\mathcal{V}|$, $N$          & Number of nodes in graph $G$. \\
        $|\mathcal{E}|$               & Number of edges in graph $G$. \\
        $\mathbf{A} \in \mathbb{R}^{N \times N}$        & Adjacency matrix of graph $G$. \\
        $\mathbf{X} \in \mathbb{R}^{N\times F}$ & Node feature matrix.  \\
        $F$ & Dimension of node feature.\\
        $\mathcal{G}$ & mm-space representation of input graph. $G$ \\
        $d_G$ & Metric of the graph $\mathcal{G}$. \\
        $\mu_G$ & Probability measure of the graph mm-space  $\mathcal{G}$. \\
        $\mathbf{q}$ &R Coordinate representation of G \\
        \midrule
        $K$ &Number of geometric bases. \\
        $\mathbf{B}_k \in \mathbb{R}^{M \times M}$ &k-th geometric base matrix.\\
        $B_k$ & mm-space representation of the $k$-th geometric base.\\ 
        $M$ & Number of nodes in geometric base. \\
        $d_k$ &The metric component of the $k$-th geometric base $B_k$. \\
        $\mu_k$ & The probability measure associated with the $k$-th geometric base $B_k$. \\
        $\mathbf{w} $ &Weight vector base on structural similarity between input graph $G$ and geometric bases. \\
        $\mathbf{T}_{ik} \in \mathbb{R}^{N\times M}$ & Optimal transport plans between graph $G_i$ and geometric base $\mathbf{B}_k$. \\
        $L$ & Number of slices in SGW computation. \\
        $\tau$ & Temperature parameter.   \\
        \midrule
        $d_{GW} (G, G')$ & GW distance between graph $G$ and $G'$. \\
        $\eta$ & Upperbound of GW distance. \\
        $\boldsymbol{\delta} \in \mathbb{R}^K$ &Consists of the GW distances from graph $G$ to the set of fixed geometric bases \\
        $\mathcal{X}$ & Space of mm-spaces equipped with the GW distance. \\
        $\mathcal{K}$ & Totally bounded subset of $\mathcal{X}$. \\
        $\mathcal{L}_{gw}$ &Loss about GW distance between graph $G$ and linear surrogate graph.\\
        $\mathcal{L}_{rec}$ & Loss of feature reconstruction. \\
        $\mathcal{L}_{div}$ & Loss of structural diversity. \\
        $m$                 & Diversity margin.        \\
        $f(\cdot)$ & 2-layer MLP decoder. \\
        $\theta$ &Decoder Parameters.   \\
        $r$ & Output Dimension of $f(\cdot)$. \\
        $L_w$ & Lipschitz constant for structural coordinates $\mathbf{w}$. \\
        $L_{sm}$ & Lipschitz constant of softmax function 
        $L_{f}$ Lipschitz constant of decoder $f(\cdot)$. \\
        $\text{FE}(\cdot)$ &Feature-extraction operator computing graph-level statistics. \\
        $\mathcal{P}$ &Represents the set of all unique pairs of bases. \\
        $\alpha$ &trade-off hyperparameter about $\mathcal{L}_{rec}$. \\
        $\beta$ &trade-off hyperparameter about $\mathcal{L}_{div}$. \\
        $\mathbf{H}_k$ &k-th hidden representation. \\
        $\mathbf{H}(G_i)$ & Projected feature matrix of graph $G_i$. \\
        $\mathbf{z}(G_i)$ & Final graph embedding. \\ 
        \bottomrule
      \end{tabular}
    \end{small}
  \end{center}
  \vskip -0.1in
\end{table}

\section{Geometric Bases Pre-training Pipeline} \label{sec:appendix_algo}
The pre-training pipeline of SCGFM is illustrated in Algorithm~\ref{alg:pretraining}.

\section{Complexity Analysis of SCGFM.}
\label{app:complexity}
We provide a more detailed analysis of SCGFM below.

\textbf{Training Complexity.}
The main compuatational bottleneck lies in metric alignment during pre-training.
We adopt \textbf{Sliced Gromov-Wasserstein (SGW)}, which reduces the dominant cost to 
\[
O\bigl(KL(N\log N + M\log M)\bigr),
\]
where $N$ and $M$ are the numbers of nodes in the input graph and geometric base, respectively, $K$ is the number of bases, and $L$ is the number of slices.
Since $K$, $L$, and $M$ are small parameters independent of graph size, the resulting scaling is \textbf{near-linear} in practice.

\textbf{Inference Complexity.}
The inference stage consists of three part:
\begin{itemize}
  \item \textbf{Structural coordinate $\mathbf{w}$ computation} via SGW: $O\bigl(KL(N\log N + M\log M)\bigr)$
  \item \textbf{Nonlinear projection} $f(\mathbf{w})$: a lightweight MLP with negligible cost $O(K^2)$.
  \item \textbf{Feature projection} $\mathbf{H}$ via entropic GW: $O(NM \cdot iter)$, where $iter$ is the number of maximum iterations in entropic GW.
\end{itemize}

Therefore, the overall inference complexity is 
\[
O\bigl(KL(N\log N + M\log M) + K^2 + NM\cdot iter\bigr).
\]

\textbf{Memory Complexity.}
SCGFM stores only a compact set of geometric bases (e.g., $16 \times 20 \times 20$), and memory usage is dominated by the input graph adjacency plus the intermediate sliced projections.
As a result, it achieves an overall memory complexity of $O(N + |E|)$. 
As further empirical evidence, the scalability experiment in Figure~\ref{fig:scalability_comparison} demonstrates that SCGFM remains stable even under an extreme stress test with \textbf{5M nodes}, reaching a peak memory usage of approximately \textbf{12 GB}. Notably, SCGFM is the only method among the compared baselines that successfully completes the 5M-node setting.

Overall, both the theoretical analysis and the stress tests confirm that SCGFM achieves favorable scalability, owing to the SGW approximation and its compact geometric base design.

\begin{algorithm}[t]
  \caption{Pre-training pipeline of SCGFM}
  \label{alg:pretraining}
  \begin{algorithmic}[1]
    \STATE {\bfseries Input:} Graph dataset $\mathcal{D} = \{G_1, \dots, G_{N}\}$, geometric bases number $K$ with $M$ nodes, slicing number $L$. Hyperparameters $\tau, \alpha, \beta$.
    \STATE {\bfseries Output:} Pre-trained geometric bases $\mathcal{B} = \{\mathbf{B}\}_{k=1}^K$ and decoder $f_{\theta}(\cdot)$.
    \STATE Initialize bases matrix $\mathbf{B}_k \in \mathbb{R}^{M \times M}$ and decoder parameters $\theta$ randomly;
    \STATE Pre-compute feature statistics $\text{FE}(G_i)$ for all $G_i \in \mathcal{D}$;
    
    \FOR{epoch $e=1, 2, \ldots, E$}
      \FOR{each mini-batch batch $\mathcal{D}_t \subset \mathcal{D}$}
        \STATE \# 1. Metric Alignment (via Sliced GW)
        \STATE Calculate SGW discrepancy for each $G_i \in \mathcal{D}_t$: $\delta_k \leftarrow \text{SGW}(\mathbf{A}_i, \mathbf{B}_k)$ via Eq.~\eqref{eq:gw_distance};
        \STATE Compute structural weights: $w_k \leftarrow \text{softmax}(-\delta_k / \tau)$ via Eq.~\eqref{eq:weights};
        
        \STATE \# 2. Surrogate Construction \& Alignment Loss
        \STATE Construct linear surrogate bases: $\tilde{\mathbf{B}}(G) \leftarrow \sum_{k=1}^K w_k \mathbf{B}_k$ via Eq.~\eqref{eq:surrogate_B};
        \STATE Calculate GW alignment loss: $\mathcal{L}_{\textbf{gw}} \leftarrow \text{SGW}(\mathbf{A}_i, \tilde{\mathcal{B}}(G))$ via Eq.~\eqref{eq:gw_loss};
        
        \STATE \# 3. Feature Reconstruction
        \STATE Decode weights to statistics: $\hat{\mathbf{s}} \leftarrow f_{\theta}(\mathbf{w})$;
        \STATE Calculate reconstruction loss: $\mathcal{L}_{\text{rec}} \leftarrow \| \text{FE}(G) - \hat{\mathbf{s}} \|^2_2$ via Eq.~\eqref{eq:rec_loss};
        
        \STATE \# 4. Optimization with Diversity
        \STATE Calculate diversity regularization: $\mathcal{L}_{\text{div}} \leftarrow \text{PairwiseDist}(\mathcal{B})$ via Eq.~\eqref{eq:div_loss};
        \STATE Total objective: $\mathcal{L}_{\text{total}} \leftarrow \mathcal{L}_{\text{gw}} + \alpha \mathcal{L}_{\text{rec}} + \beta \mathcal{L}_{\text{div}}$;
        \STATE Update $\mathcal{B}$ and $\theta$ by minimizing $\mathcal{L}_{\text{total}}$;
      \ENDFOR
    \ENDFOR
  \end{algorithmic}
\end{algorithm}

\subsection{Scalability with Respect to Geometric Bases}

We further discuss the computational bottlenecks of SCGFM when the number of geometric bases or the basis size increases.
As reported in the main paper, SCGFM shows strong scalability in both the scalability analysis and the sensitivity studies.
Here we clarify this issue from theoretical and empirical perspectives.

First, the dominant training cost scales as
\begin{equation}
O\bigl(KL(N\log N + M\log M)\bigr),
\end{equation}
where $K$ denotes the number of geometric bases, $M$ denotes the basis size, $N$ denotes the input graph size, and $L$ denotes the number of sliced projections.
This indicates that increasing either $K$ or $M$ introduces approximately linear overhead in practice.

Second, the geometric bases are designed to serve as compact structural abstractions rather than large graph templates.
When $M \ll N$, each basis acts as a compressed geometric primitive or motif that summarizes representative topological patterns.
Increasing $M$ beyond a moderate range weakens this compression effect and may introduce redundant high-dimensional templates.

Third, our empirical results show that performance saturates once $K$ and $M$ reach modest values.
To further validate this observation, we conduct a stress study on COLLAB by increasing either the number of bases or the bases size.
The results are reported in Table~\ref{tab:base_scalability}.

\begin{table}[h]
\centering
\caption{Scalability study of SCGFM with different numbers of bases $K$ and basis sizes $M$ on COLLAB.}
\label{tab:base_scalability}
  \begin{tabular}{l c c c c}
    \hline
    Setting & Variant & Time/Epoch (s) & Memory (GB) & 5-shot Acc. (\%) \\
    \hline
    \shortstack{Default batch size $=128$}
    & $K=16, M=32$ & 5.12 $\pm$ 0.66 & 0.23 & 66.40 $\pm$ 6.33 \\
    \hline
    \shortstack{Scale $K$ (fix $M=32$)}
    & $K=32$ & 5.77 $\pm$ 0.38 & 0.34 & 66.12 $\pm$ 6.12 \\
    & $K=64$ & 5.76 $\pm$ 0.35 & 0.61 & 63.83 $\pm$ 6.17 \\
    \hline
    \shortstack{Scale $M$ (fix $K=16$)}
    & $M=64$ & 4.85 $\pm$ 0.11 & 0.25 & 65.68 $\pm$ 5.07 \\
    & $M=128$ & 4.80 $\pm$ 0.12 & 0.37 & 64.53 $\pm$ 4.69 \\
    \hline
  \end{tabular}
\end{table}

The results show that increasing the bases size to approximately four times the default configuration only introduces modest memory overhead, while the accuracy improvement is negligible or slightly negative.
Similarly, increasing the number of bases does not consistently improve performance.
These observations are consistent with the saturation trend reported in the main paper, suggesting that a compact geometric dictionary is sufficient to capture the structural diversity needed for effective transfer.

\section{Datasets} \label{datasets}
Detailed datasets is illustrated in Table~\ref{tab:data}
\begin{table}[ht]
  \caption{Statistics of Multi-Domain Graph Datasets. NONE indicates the absence of input node features.}
  \label{tab:data}
  \centering
  \resizebox{\textwidth}{!}{
    \begin{small}
        \begin{tabular}{lccccccc}
          \toprule
          \textbf{Dataset} & \textbf{Domain} & \textbf{Graphs} & \textbf{\# Nodes} & \textbf{\# Edges} & \multicolumn{1}{c}{\textbf{\# Feature}} & \textbf{\# Classes} \\
          & & &\textbf{(or AVG.)} &\textbf{or AVG.} & \textbf{Dimensions} & \\
          \midrule
          COX2~\cite{sutherland2003spline} & Molecules &467 &41.22 &43.45 &3 & 2 \\
          NCI1~\cite{wale2008comparison} & Molecules &4110 & 29.87 &32.30 & NONE& 2 \\
          BZR~\cite{sutherland2003spline}  & Molecules &405 &35.75 &38.36 & 3& 2 \\
          MUTAG~\cite{debNath1991structure} & Molecules &188 & 17.93 & 19.79&NONE &2   \\
          \midrule
          COLLAB~\cite{10.1145/2783258.2783417} & Social Networks &5000 &74.49 &2457.78 &NONE &3  \\
          IMDB-BINARY~\cite{10.1145/2783258.2783417} & Social Networks & 1000&19.77 & 96.53& NONE&  2\\
          \midrule
          COLORS-3~\cite{10.5555/3454287.3454665} & Synthetic & 10500&61.31 &91.03 &4 &11\\
          \midrule
          PROTEINS~\cite{10.1093/bioinformatics/bti1007} & Bioinformatics & 1113&39.06 & 72.82& 1& 2 \\
          \midrule
          Cora~\cite{pmlr-v48-yanga16} & Academic & 1 & 2708 & 10,556 & 1,433 & 7  \\
          CiteSeer~\cite{pmlr-v48-yanga16} & Academic & 1 & 3,327 & 9104 & 3,703 & 6 \\
          PubMed~\cite{pmlr-v48-yanga16} & Academic & 1 & 19,717 & 88,648 & 500 & 3  \\
          \midrule
          Photo~\cite{Shchur2018PitfallsOG} & E-Commerce & 1 & 7,650 & 238,162 & 745 & 8  \\
          Computers~\cite{Shchur2018PitfallsOG} & E-Commerce & 1 & 13,752 & 491,722 & 767 & 10 \\
          \midrule
          Reddit~\cite{10.5555/3294771.3294869} & Social Networks & 1 & 232,965 & 114,615,892 & 602 & 41  \\
          \bottomrule
        \end{tabular}
    \end{small}
  }
  \vskip -0.1in
\end{table}

\section{Experimental Setting Details} \label{ap:experiment_setup}

\subsection{Few-shot Learning Detailed Settings.}

In the few-shot pretraining stage, all self-supervised models with configurable backbones adopt a
two-layer GCN as the encoder to ensure architectural consistency across different methods.
During the downstream few-shot evaluation, the pretrained encoder is frozen, and no further fine-tuning is performed.
Instead, a prototypical network is employed as the sole classifier.
This design choice allows us to isolate the quality of the representations learned during
pretraining, eliminating the influence of task-specific adaptation.

All downstream few-shot tasks follow a unified evaluation protocol with \textbf{5-shot support sets,
50 query samples per class, and results averaged over 50 independent runs}.
Such a strict few-shot setting is intentionally adopted to emphasize the transferability of
cross-domain knowledge learned during pretraining, rather than the model's capacity to adapt
through gradient-based updates. Moreover, using an identical evaluation protocol across all
methods ensures a fair and reproducible comparison under limited supervision.

For node classification tasks, we transform the problem into a graph-level few-shot setting by constructing \textbf{ego-centric subgraphs using Personalized PageRank (PPR) sampling}.
Specifically, for each target node, we sample a subgraph containing up to 100 nodes, and assign the label of the central node as the label of the corresponding subgraph.
Importantly, the PPR-based subgraph sampling is performed \emph{once} to construct a fixed dataset, which is shared by all compared methods.
By ensuring that all models are trained and evaluated on exactly the same sampled subgraphs, we eliminate potential performance variations caused by stochastic sampling, thereby guaranteeing a fair and controlled comparison across different approaches.

\subsection{Hyperparameters Setting In Few-shot Experiments}
For the node classification task, we use a fixed set of hyperparameters across all experiments. Specifically, the number of geometric bases $K$ is set to $16$, $M$ to 32, and the number of layers to 50. The learning rate is 0.01, with temperature $\tau=0.3$, and trade-off coefficients $\alpha=2$ and $\beta =0.05$.
We set the diversity margin to 10, train the model for 60 epochs, and use a batch size of 32.
The random seed was fixed to 42 across all experiments.

For graph classification tasks on COX2, NCI1, BZR, COLLAB, IMDB-BINARY, and COLORS-3, most hyperparameters follow the same configuration as the node classification task. The main differences lie in $M$, which is set to 30 for COX2, NCI1, BZR, and COLORS-3, 20 for COLLAB, and 6 for IMDB-BINARY.
In addition, the diversity margin is adjusted to 8 for COLLAB and 3 for IMDB-BINARY, while all other hyperparameters remain unchanged.

\subsection{Hardware and Software}
We conduct all experiments on the following configurations:
\begin{itemize}
  \item \textbf{Operating System}: Ubuntu 22.04 LTS.
  \item \textbf{CPU}: 13th Gen Intel(R) Core(TM) i9-13900K.
  \item \textbf{GPU}:  NVIDIA GeForce RTX 3090 with 24GB of memory.
  \item \textbf{Software}: CUDA 12.1, Python 3.10.19, Pytorch\footnote{\url{https://github.com/pytorch/pytorch}} 2.5.1, Pytorch Geometric\footnote{\url{https://github.com/pyg-team/pytorch_geometric}} 2.7.0
\end{itemize}

\section{Additional Experiments}

\subsection{Hyperparameter Sensitivity Analysis}
\label{ap:hyperparameter}

To evaluate the stability of SCGFM, we conducted extensive sensitivity analysis on five key hyperparameters using the PROTEINS dataset under the 2-way 5-shot setting.
The results, illustrated in Figure~\ref{fig:hyperparams}, demonstrate that our method maintains high performance across a wide range of configurations.

\begin{figure}[ht]
    \centering
    \includegraphics[width=1.0\linewidth]{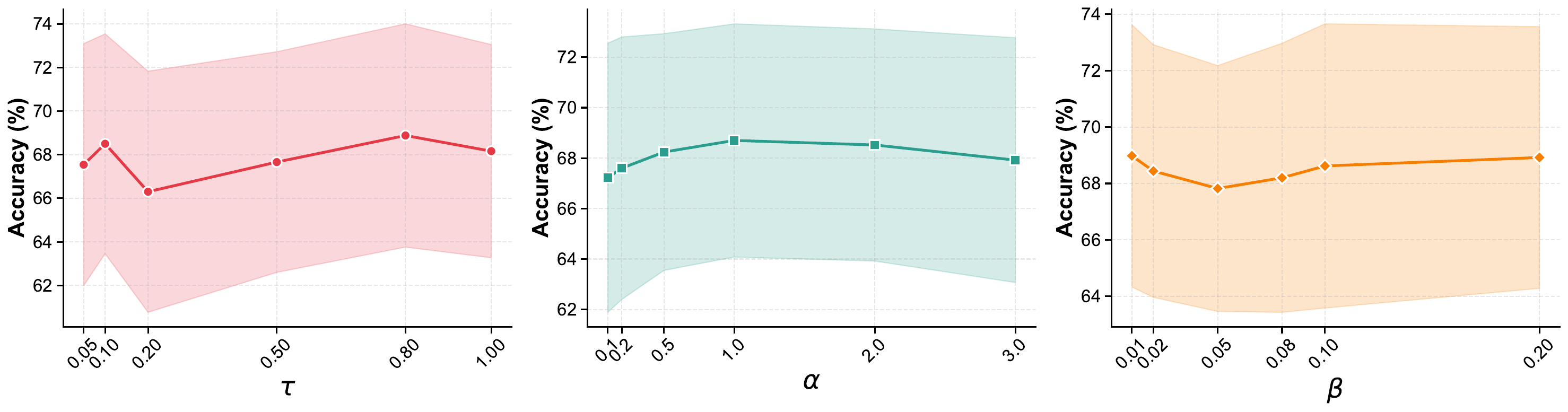}
    \caption{Sensitivity analysis of key hyperparameters on PROTEINS classification accuracy. The model exhibits strong robustness, with performance remaining stable across broad ranges of parameter choices.}
    \label{fig:hyperparams}
\end{figure}

\textbf{Optimization \& Regularization ($\tau, \alpha, \beta$)}:
\begin{itemize}
    \item \textbf{Temperature ($\tau$):} This parameter controls the sharpness of the attention weights in the bases matching step (Eq. \ref{eq:weights}).
    The plot indicates an optimal range around $\tau \in [0.5, 0.8]$.
    Extreme sharpness ($\tau \to 0$) or excessive smoothness ($\tau \to 1.0$) slightly degrades performance, but the overall variance is contained within $2\%$, indicating the key matching mechanism is robust.
    
    \item \textbf{Loss Coefficients ($\alpha$ and $\beta$):} We denote the weighting coefficients for reconstruction loss and diversity regularization as $\alpha$ and $\beta$, respectively.
    \begin{itemize}
        \item $\alpha$ shows a "sweet spot" around $1.0$, balancing the structural separation (GW loss) and graph-level statistics reconstruction.
        \item $\beta$ (diversity weight) exhibits a very flat trend, maintaining accuracy between $68\%$ and $69\%$ across orders of magnitude ($0.01$ to $0.2$).
        This confirms that while diversity is helpful for preventing mode collapse, the model is not hypersensitive to the exact strength of this regularization.
    \end{itemize}
\end{itemize}

\paragraph{Conclusion.} 
Overall, SCGFM demonstrates \textbf{strong robustness}.
The method does not require meticulous hyperparameter tuning to achieve SOTA-level performance, making it highly practical for diverse real-world graph tasks.

\subsection{Additional Ablation Study} \label{ap:add_ab_ge}
In this section, we provide a granular analysis of the ablation study regarding the final graph representation $\mathbf{z}(G)$.
According to Eq.~\ref{eq:graph_embedding}, the representation is composed of three distinct parts: the geometric coordinate vector $\mathbf{w}$ (derived from OT-based projection), its non-linear projection $f(\mathbf{w})$, and the \textbf{structure-aligned feature representation} $\text{vec}(\mathbf{H})$ (derived from the structure-aware re-encoding in Eq.~\ref{eq:project_h}).
Figure~\ref{fig:detailed_ablation} reveal a fundamental dichotomy in graph representation learning, demonstrating how SCGFM effectively disentangles geometry from semantics.

\begin{figure}[ht]
    \centering
    \includegraphics[width=1.0\linewidth]{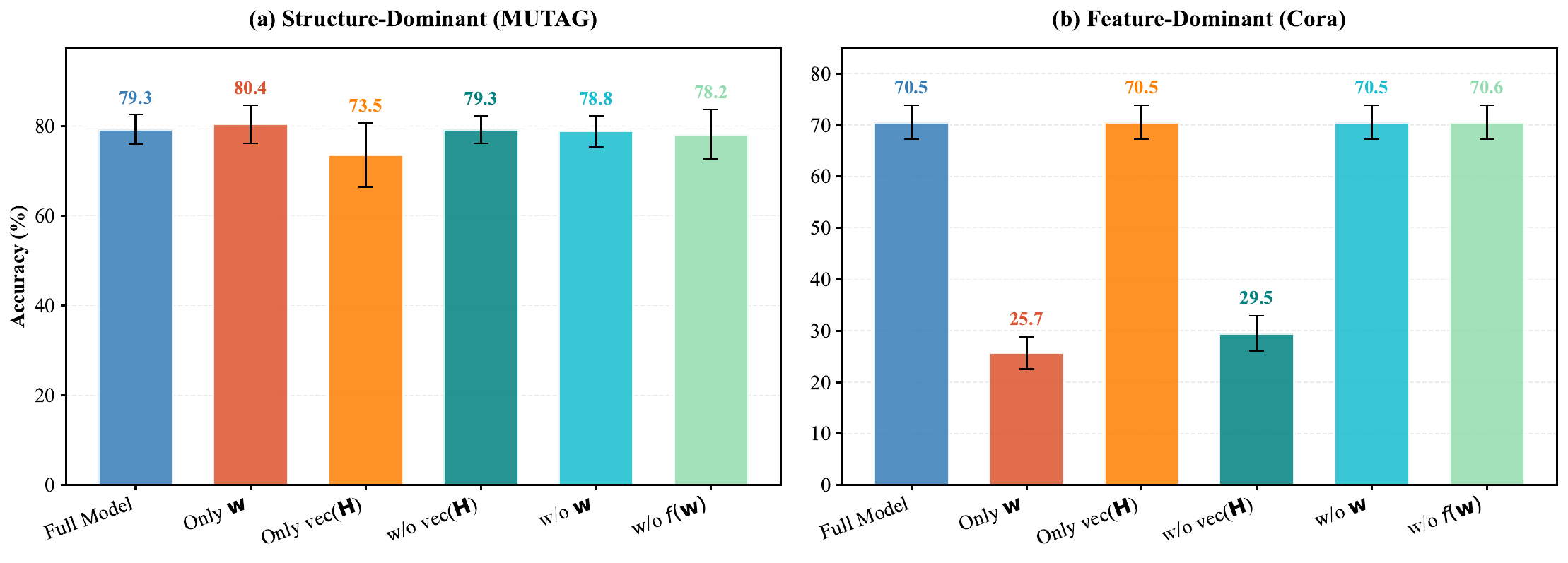}
    \caption{Ablation study on the impact of structural and semantic components.
    We evaluate the model variants on two representative datasets: (a) MUTAG (structure-dominant) and (b) Cora (feature-dominant).
    $\mathbf{w}$ denotes the learnable structure-related parameter, and $\operatorname{vec}(\mathbf{H}(G))$ denotes the semantic feature vectors.
    Standard deviation is indicated by error bars.
    The results show that the SCGFM effectively leverages the dominant modality in each scenario, whereas decoupling the components reveals the dataset-specific reliance on structure versus semantics.}
    \label{fig:detailed_ablation}
\end{figure}

\subsubsection{The "Feature-Dominant" Regime: The Case of Cora}
Cora is a citation network where classification relies heavily on document content (node features).
\begin{itemize}
    \item \textbf{Failure of Pure Geometry ($\mathbf{w}$):}
    As observed in the ``Only $\mathbf{w}$'' variant, relying solely on geometric coefficients results in a catastrophic performance drop (Accuracy $\approx 25\%$).
    This is expected because $\mathbf{w}$ only captures the topological similarity to the bases.
    Since the citation graphs of different topics often share isomorphic structures (e.g., star-like patterns), the geometric coordinates alone cannot distinguish between classes without semantic content.
    \item \textbf{Dominance of Aligned Features ($\text{vec}(\mathbf{H})$):}
    The ``Only $\text{vec}(\mathbf{H})$'' variant achieves $70.55\%$, matching the Full model.
    Here, $\text{vec}(\mathbf{H})$ represents the input node features (Bag-of-Words) aggregated onto the geometric bases via the transport plan $\mathbf{T}$.
    This result confirms that preserving the feature content is paramount for citation networks, even after projecting them into the shared latent space.
    \item \textbf{Role of SCGFM:}
    The Full model successfully identifies that the semantic signal in $\text{vec}(\mathbf{H})$ is the primary discriminator, while implicitly suppressing the noise from the ambiguous geometric signal $\mathbf{w}$ for this specific task.
\end{itemize}

\subsubsection{The "Structure-Dominant" Regime: The Case of MUTAG}
MUTAG is a molecular dataset where the task is to predict mutagenicity, a property inherently tied to molecular topology (e.g., rings, connectivity).
\begin{itemize}
    \item \textbf{Limitation of Pure Features:}
    The ``Only $\text{vec}(\mathbf{H})$'' variant achieves $73.50\%$.
    Although these features are structurally aligned, they primarily encode atom types (e.g., C, N, O) mapped to bases.
    In molecular tasks, the \textit{existence} of a specific pattern (captured by $\mathbf{w}$) is often more critical than the \textit{content} at that pattern (captured by $\text{vec}(\mathbf{H})$).
    \item \textbf{Superiority of Geometric Coordinates ($\mathbf{w}$):}
    The ``Only $\mathbf{w}$'' variant achieves a remarkable $80.42\%$, significantly outperforming the feature-only baseline.
    This implies that the transport cost vector $\mathbf{w}$—which measures how much effort is needed to morph the input graph into each geometric base—acts as a powerful topological descriptor (fingerprint) for molecular properties.
\end{itemize}

\subsubsection{The Role of Non-Linear Projection $f(\mathbf{w})$}
A natural question arises regarding the necessity of the projection head $f(\cdot)$ in Eq.~\ref{eq:graph_embedding}, given that the raw coordinates $\mathbf{w}$ already contain strong structural signals.
Based on Table~\ref{fig:detailed_ablation}, we analyze its function:
\begin{itemize}
    \item \textbf{Latent Space Adaptation:}
    The geometric bases count $K$ is typically small (e.g., $K=16$) to ensure orthogonality, whereas the semantic vector $\text{vec}(\mathbf{H})$ usually has a much higher dimensionality.
    Simply concatenating a low-dim $\mathbf{w}$ with a high-dim $\text{vec}(\mathbf{H})$ would cause the geometric signal to be overwhelmed.
    $f(\cdot)$ lifts $\mathbf{w}$ to a compatible dimension ($r$), balancing the contribution of geometry and semantics in the final embedding $\mathbf{z}(G)$.
    \item \textbf{Non-Linear Refinement:}
    While $\mathbf{w}$ represents linear transport costs, $f(\mathbf{w})$ (parameterized as an MLP) captures non-linear interactions between these costs.
    Comparing the \textit{Full Model} ($79.30\%$) with \textit{w/o $f(\mathbf{w})$} ($78.18\%$) on MUTAG, we observe a consistent performance gain ($+1.12\%$).
    This indicates that the non-linear transformation extracts specific geometric combinations critical for classification that are not present in the raw coordinates.
    \item \textbf{Redundancy vs. Robustness:}
    Interestingly, on some datasets, $\mathbf{w}$ and $f(\mathbf{w})$ appear redundant (e.g., purely structural performance is similar).
    However, retaining $f(\mathbf{w})$ ensures the model's robustness.
    It acts as a buffer that allows the network to learn deeper geometric semantics when raw coordinates are insufficient, without hurting performance in simpler cases.
\end{itemize}

\subsubsection{Conclusion: Geometric and Semantic Complementarity}
The study validates the design of Eq.~\ref{eq:graph_embedding}:
\begin{enumerate}
    \item $\mathbf{w}$ answers \textbf{``Where does the graph lie in the geometric manifold?''} (crucial for Structure-heavy tasks).
    \item $\text{vec}(\mathbf{H})$ answers \textbf{``What content does the graph carry at those coordinates?''} (crucial for Feature-heavy tasks).
\end{enumerate}
By concatenating them, SCGFM avoids the trade-off inherent in single-view models, achieving robust performance across both regimes.

\subsection{Versatility and Robustness of Representations across Domains}
\label{appendix:classifier_robustness}
To verify that the structural representations learned by SCGFM are versatile and not strictly tied to a specific metric-based classifier, we extend our evaluation to both \textbf{Graph Classification} (PROTEINS) and \textbf{Node Classification} (Cora). 

We freeze the SCGFM encoder and evaluate five distinct classifiers.
This setup tests whether the embeddings contain sufficient information to support different downstream decision mechanisms (metric-based vs. boundary-based).

\begin{table}[h]
\centering
\caption{Classifier performance comparison across domains.
  \textbf{Left:} Graph Classification on PROTEINS.
  \textbf{Right:} Node Classification on Cora (High-Dim features $\approx$ 28k).
  PN demonstrates consistent robustness.
  }
\label{tab:classifier_cross_domain}
  \begin{small}
  \begin{tabular}{lcccc}
    \toprule
    & \multicolumn{2}{c}{\textbf{PROTEINS (Graph)}} & \multicolumn{2}{c}{\textbf{Cora (Node, High-Dim)}} \\
    \cmidrule(lr){2-3} \cmidrule(lr){4-5}
    \textbf{Method} & \textbf{Accuracy (\%)} & \textbf{Std ($\pm$)} & \textbf{Accuracy (\%)} & \textbf{Std ($\pm$)} \\
    \midrule
    \textbf{PN} & \textbf{68.46} & 5.15 & \textbf{70.55} & 3.27 \\
    \textbf{Random Forest} & 66.08 & 5.87 & 70.37 & 3.45 \\
    \textbf{MLP} & 64.16 & 6.43 & 65.99 & 4.00 \\
    \textbf{$k$-NN} & 65.04 & 7.85 & 61.41 & 4.91 \\
    \textbf{SVM} & 64.60 & 5.74 & 59.40 & 5.02 \\
    \bottomrule
  \end{tabular}
  \end{small}
\end{table}

\paragraph{Analysis of Results.}
Table~\ref{tab:classifier_cross_domain} reveals two key insights regarding the quality of the learned representations:

\textbf{1. Robustness to High Dimensionality (PN).}
On the Cora dataset, where the feature dimension expands to approx. 28k, distance-sensitive methods like $k$-NN and SVM suffer from the curse of dimensionality ($61.41\%$ and $59.40\%$, respectively).
However, PN adapts well ($70.55\%$) by averaging support samples into prototypes, effectively suppressing high-dimensional noise.
This confirms that the SCGFM embeddings maintain a stable global geometry even in extreme high-dimensional settings.

\textbf{2. Versatility Beyond Metrics (Random Forest).}
Crucially, \textbf{Random Forest (RF)} achieves performance nearly identical to PN on Cora ($70.37\%$ vs. $70.55\%$).
Unlike PN, RF does not rely on Euclidean distance but on feature splitting.
This high performance proves that SCGFM embeddings are \textbf{discriminative independent of the distance metric}.
The encoder successfully disentangles complex structural information into individual feature dimensions, allowing non-metric classifiers to extract meaningful patterns easily.

\paragraph{Conclusion.}
The experiment confirms that SCGFM produces \textbf{task-agnostic and versatile embeddings}.
While PN is the most robust choice for few-shot metric learning tasks, the strong performance of Random Forest indicates that the learned representations are rich in semantic content and can be effectively utilized by a broad range of downstream applications, largely unaffected by classifier-specific constraints.

\subsection{Evaluation Beyond ProtoNet-based Classification}

To examine whether SCGFM learns generally useful graph representations beyond ProtoNet-based few-shot classification, we evaluate the frozen embeddings on two additional tasks: supervised adaptation with a lightweight MLP classifier and graph retrieval with $k$NN ranking.
The encoder is pretrained on PROTEINS and transferred to MUTAG and COX2.

\paragraph{Supervised adaptation.}
We freeze the pretrained encoder and train an MLP classifier with limited labeled data on the target dataset.
The results are shown in Table~\ref{tab:supervised_adaptation}.
SCGFM achieves competitive or better performance than standard GNN baselines (GCN, GIN, GAT) under the same low-label setting.

\begin{table}[h]
\centering
\caption{Supervised adaptation with a frozen encoder and an MLP classifier. The encoder is pretrained on PROTEINS and transferred to MUTAG and COX2.}
\label{tab:supervised_adaptation}
  \begin{tabular}{l c c c}
    \hline
    Dataset & Labels & SCGFM & Best GNN \\
    \hline
    MUTAG & 10\% & \textbf{83.78} & 81.58 (GIN) \\
    MUTAG & 20\% & \textbf{83.78} & 78.95 (GIN) \\
    COX2  & 10\% & 80.65 & \textbf{81.91} (GCN) \\
    COX2  & 20\% & \textbf{87.10} & 81.91 (GCN) \\
    \hline
  \end{tabular}
\end{table}

\paragraph{Graph retrieval.}
We further evaluate the learned embeddings in a graph retrieval setting.
Given a query graph, database graphs are ranked according to their embedding distances.
We report Precision@10, MAP@10, and $k$NN classification accuracy at $k=10$.
The results are presented in Table~\ref{tab:graph_retrieval}.
SCGFM consistently outperforms GCN, GAT, and GIN on both MUTAG and COX2, indicating that the learned geometric embedding space preserves meaningful neighborhood structure for retrieval.

\begin{table}[h]
\centering
\caption{Graph retrieval performance using frozen embeddings. Higher values indicate better performance.}
\label{tab:graph_retrieval}
  \begin{tabular}{l l c c c}
    \hline
    Dataset & Model & P@10 $\uparrow$ & MAP@10 $\uparrow$ & kNN-Acc@10 $\uparrow$ \\
    \hline
    MUTAG & SCGFM & \textbf{0.7840} & \textbf{0.7162} & \textbf{0.8351} \\
          & GIN   & 0.7356 & 0.6659 & 0.7979 \\
          & GCN   & 0.6686 & 0.5822 & 0.7394 \\
          & GAT   & 0.6543 & 0.5511 & 0.7340 \\
    \hline
    COX2  & SCGFM & \textbf{0.7015} & \textbf{0.6068} & \textbf{0.7923} \\
          & GCN   & 0.6754 & 0.5703 & 0.7880 \\
          & GAT   & 0.6739 & 0.5676 & 0.7794 \\
          & GIN   & 0.6666 & 0.5516 & 0.7730 \\
    \hline
  \end{tabular}
\end{table}

These results show that SCGFM embeddings are not limited to ProtoNet-based classification, but also support supervised adaptation and retrieval.

\subsection{Additional Scalability Analysis}
As illustrated in Figure~\ref{fig:scalability_comparison}, we conducted comparative experiments on synthetic datasets with a fixed batch of 1,000 graphs, while varying the number of nodes per graph.
SCGFM demonstrates superior stability as it scales to massive graphs.
The raw data are provided below:
\begin{table}[h]
\centering
  \caption{Runtime comparison on synthetic graphs with increasing graph sizes.}
  \label{tab:runtime_scalability}
  \begin{tabular}{c c c c c c c}
    \hline
    Node Number & Average Edges & GAT & GIN & GIT & RiemannGFM & SCGFM \\
    \hline
    0.1M & 494.67 & 0.12s & 0.10s & 0.21s & 0.24s & 0.84s \\
    0.5M & 12468.16 & 0.70s & 0.19s & 1.15s & 0.54s & 0.89s \\
    1.0M & 49954.51 & 2.61s & 0.61s & 3.73s & 1.46s & 1.47s \\
    3.0M & 449841.264 & OOM & 4.81s & OOM & 8.46s & 7.69s \\
    5.0M & 1249773.292 & OOM & OOM & OOM & OOM & \textbf{19.96s} \\
    \hline
  \end{tabular}
\end{table}

These results indicate that SCGFM incurs a slightly higher constant overhead on small graphs due to metric alignment.
However, its runtime scales much more gracefully with increasing graph size, and it is the only method that remains executable on 5.0M nodes within a 24 GB VRAM budget.
In contrast, all other advanced baselines encounter out-of-memory failures before reaching this scale.

\subsection{Additional Cross-Domain GFM Baselines}

To further evaluate the cross-domain transferability of SCGFM, we include three recent GFM or graph pretraining baselines that are designed for cross-domain or heterogeneous graph scenarios: SAMGPT, BRIDGE, and MDGFM.
All methods are evaluated under the same 5-shot cross-domain graph classification setting.

Since these methods adopt different downstream adaptation strategies, we report two complementary settings for fair comparison.
The \emph{native} setting preserves each method's original target-side prompt or adaptation mechanism whenever supported by the released code.
The \emph{aligned} setting freezes the pretrained encoder as much as possible, minimizes target-side adaptation, and evaluates all methods using ProtoNet under the same fixed-encoder protocol as SCGFM.

\begin{table}[h]
\centering
\small
\caption{Additional 5-shot cross-domain graph classification results with recent GFM and graph pretraining baselines.}
\label{tab:additional_gfm_baselines}
  \begin{tabular}{l c c}
    \hline
    Model & BZR $\rightarrow$ PROTEINS (\%) & PROTEINS $\rightarrow$ BZR (\%) \\
    \hline
    BRIDGE (native) & 51.06 $\pm$ 5.80 & 56.38 $\pm$ 6.89 \\
    SAMGPT (native) & 51.36 $\pm$ 5.52 & 57.90 $\pm$ 7.47 \\
    MDGFM (native) & 57.06 $\pm$ 7.70 & 53.00 $\pm$ 6.02 \\
    BRIDGE (aligned) & 49.94 $\pm$ 0.42 & 51.24 $\pm$ 5.94 \\
    SAMGPT (aligned) & 50.16 $\pm$ 0.89 & 53.12 $\pm$ 7.08 \\
    MDGFM (aligned) & 54.04 $\pm$ 6.78 & 51.94 $\pm$ 5.30 \\
    SCGFM & \textbf{68.32 $\pm$ 4.65} & \textbf{58.60 $\pm$ 6.75} \\
    \hline
  \end{tabular}
\end{table}

As shown in Table~\ref{tab:additional_gfm_baselines}, SCGFM achieves the best performance in both transfer directions.
It also outperforms the strongest native baseline, demonstrating its stronger robustness under heterogeneous cross-domain transfer.

\subsection{Baseline Selection and Evaluation Protocol}

We clarify the rationale behind our baseline selection and evaluation protocol.
Recent GFMs such as OFA, SAMGPT, BRIDGE, MDGFM, and RAG4GFM are important related methods, but many of them follow an \emph{interface-centric} paradigm, where prompts, instructions, retrieval modules, or task-specific target-side adaptation are used to bridge different graph domains.
In contrast, SCGFM is designed as a \emph{structure-centric} model: it constructs a shared geometric coordinate system directly from graph topology and transfers through structure-aligned representations, without relying on textual semantics or prompt-based interfaces.

This distinction also affects the evaluation protocol.
In the main paper, we evaluate SCGFM under a strict \emph{fixed-encoder few-shot} setting, where the pretrained model is frozen and only a lightweight downstream classifier is used.
This protocol is intended to measure the intrinsic transferability and representation quality of the pretrained encoder.
Prompt-based GFMs, however, are often designed to exploit target-side prompts, adapters, or retrieval-based augmentation during transfer.
Directly comparing their native target-adaptation pipelines with a frozen SCGFM encoder would therefore conflate representation quality with additional adaptation mechanisms.

To address this issue, we distinguish between two complementary settings in the supplementary comparison (Table~\ref{tab:additional_gfm_baselines}).
These comparisons complement the main results and show that SCGFM remains competitive, and in most cases superior, under both native and protocol-aligned cross-domain evaluation settings.

\subsection{Interpretability and Base Specialization}

We further analyze how the interpretability and specialization of geometric bases change with larger basis size $M$, larger number of bases $K$, and more diverse pretraining domains.
We provide quantitative evidence from three perspectives.

\paragraph{Base specialization across diverse domains.}
We jointly pretrain SCGFM on PROTEINS from bioinformatics and IMDB-BINARY from social networks using $K=8$ bases, without providing any domain labels.
We then compute the average activation weight $w_k$ of each base on graphs from the two domains.
The results are summarized in Table~\ref{tab:domain_specialization}.

\begin{table}[h]
\centering
\caption{Domain-dependent activation of learned geometric bases. The uniform activation weight is $1/8=0.125$.}
\label{tab:domain_specialization}
  \begin{tabular}{l c c c}
    \hline
    Domain & $B_7$ Bio. & $B_6$ Soc. & $B_2/B_8$ Shared \\
    \hline
    PROTEINS & \textbf{0.1803} & 0.1265 & $\sim$0.178 \\
    IMDB-BINARY & 0.1023 & \textbf{0.1719} & $\sim$0.166 \\
    Shift & 1.76$\times$ higher & 1.36$\times$ higher & Domain-agnostic \\
    \hline
  \end{tabular}
\end{table}

The domain-specific shifts are statistically significant $p<10^{-6}$.
Bases $B_7$ and $B_6$ show clear domain preferences, corresponding respectively to bioinformatics-related and social-network-related structural motifs.
In contrast, $B_2$ and $B_8$ maintain consistently high activation in both domains, suggesting that they capture universal geometric primitives shared across domains.
These results indicate that SCGFM can naturally disentangle domain-specific motifs from domain-agnostic topological patterns without supervision.

\paragraph{Scaling the base size $M$.}
We also examine how the topology of learned bases changes as the base size $M$ increases.
Table~\ref{tab:base_size_interpretability} reports the average degree, number of connected components, and qualitative topology focus of the learned bases.

\begin{table}[h]
\centering
\caption{Topological complexity of learned bases under different base sizes $M$.}
\label{tab:base_size_interpretability}
  \begin{tabular}{c c c l}
    \hline
    $M$ & Avg. Degree & Avg. Comp. & Topology Focus \\
    \hline
    8  & 2.61 & 1.62 & Simple primitives \\
    16 & 4.71 & 1.06 & Connected subgraphs \\
    24 & 6.60 & 1.04 & Complex motifs \\
    \hline
  \end{tabular}
\end{table}

As $M$ increases, the learned bases exhibit higher structural resolution.
Small bases mainly capture simple primitives such as edges and star-like patterns.
With larger $M$, the bases gradually become more connected and encode more complex multi-hop motifs, such as clique-like or fused-ring structures.
This suggests that increasing $M$ improves the expressive capacity of individual bases.

\paragraph{Scaling the number of bases $K$.}
We further study how the roles of bases differentiate as $K$ increases.
For each test graph, we identify its top-1 base, i.e., the base with the largest activation weight $w_k$.
We then count how many distinct bases are selected as top-1 by at least one graph.
The results are shown in Table~\ref{tab:base_number_interpretability}.

\begin{table}[h]
\centering
\caption{Role differentiation of bases as the number of bases $K$ increases.}
\label{tab:base_number_interpretability}
  \begin{tabular}{c c c}
    \hline
    $K$ & Unique Top-1 / $K$ & Usage Ratio \\
    \hline
    4  & 3 / 4   & 75\% \\
    8  & 4 / 8   & 50\% \\
    16 & 9 / 16  & 56\% \\
    24 & 8 / 24  & 33\% \\
    \hline
  \end{tabular}
\end{table}

As $K$ increases from 4 to 24, the fraction of bases serving as the primary prototype decreases from 75\% to 33\%.
This indicates that the learned vocabulary spontaneously differentiates into two groups: a small set of core prototypes that dominate primary graph matching, and a larger set of auxiliary bases that provide finer structural refinements through soft combinations.
This unsupervised role differentiation improves interpretability, since practitioners can inspect the core prototypes to identify dominant topological motifs while using auxiliary bases to explain fine-grained variations.

\paragraph{Connection to downstream performance.}
The interpretability trends are consistent with the performance sensitivity results in Figure~\ref{fig:hyperparameter} of the main paper.
Downstream accuracy remains stable, and in some cases improves, as $M$ and $K$ increase within the studied range.
This suggests that greater base specialization improves representation quality rather than merely introducing redundant bases.

\section{Discussion}

\subsection{Representation Components and Structure-conditioned Feature Projection}
\label{app:component_analysis}

SCGFM uses a combined representation
\[
    Z(G) = [\mathbf{w}, f(\mathbf{w}), \mathrm{vec}(\mathbf{H})],
\]
where $\mathbf{w}$ denotes the structural coordinates, $f(\mathbf{w})$ denotes the decoder output, and $\mathrm{vec}(\mathbf{H})$ denotes the vectorized structure-conditioned feature representation.
This design allows the representation to adapt to both structure-dominant and feature-dominant regimes.

Importantly, $\mathrm{vec}(\mathbf{H})$ is not an arbitrary high-dimensional lift of node features.
It is conditioned on graph structure through the OT coupling $\mathbf{T}$ in Eq.~\ref{eq:project_h}, which aligns the input graph with the learned geometric bases.
Thus, $\mathrm{vec}(\mathbf{H})$ reflects the feature information after structural alignment.

To verify this, we conduct a diagnostic experiment on 1000 Cora ego-graphs.
We fix the node feature matrix $\mathbf{X}$ and perturb only the graph topology $\mathbf{A}$ by rewiring edges with probability $\epsilon$.
We then measure the cosine distance of the representation before and after rewiring.
As a same-dimensional structure-agnostic control, we use \texttt{pool\_tiled}, which preserves dimensionality but removes OT-based structural conditioning.
The results are shown in Table~\ref{tab:rewire_diagnostic}.

\begin{table}[h]
\centering
\small
\caption{Diagnostic study with fixed node features and rewired graph structure on 1000 Cora ego-graphs. $\Delta$ denotes cosine distance before and after rewiring.}
\label{tab:rewire_diagnostic}
  \begin{tabular}{c c c c}
    \hline
    Rewire $\epsilon$
    & $\Delta \mathrm{vec}(\mathbf{H})$ Full
    & $\Delta \mathrm{vec}(\mathbf{H})$ \texttt{pool\_tiled}
    & $\Delta \mathbf{w}$ Full \\
    \hline
    0.00
    & $1.0{\times}10^{-8} \pm 7.0{\times}10^{-8}$
    & $1.0{\times}10^{-8} \pm 8.0{\times}10^{-8}$
    & $\sim 0$ \\
    0.30
    & $1.86{\times}10^{-2} \pm 3.25{\times}10^{-2}$
    & $1.0{\times}10^{-8} \pm 8.0{\times}10^{-8}$
    & $7.9{\times}10^{-7} \pm 8.9{\times}10^{-7}$ \\
    0.70
    & $5.26{\times}10^{-2} \pm 7.04{\times}10^{-2}$
    & $1.0{\times}10^{-8} \pm 8.0{\times}10^{-8}$
    & $2.2{\times}10^{-6} \pm 2.3{\times}10^{-6}$ \\
    \hline
  \end{tabular}
\end{table}

The full model's $\mathrm{vec}(\mathbf{H})$ changes monotonically with structural perturbation, whereas the same-dimensional structure-agnostic control remains nearly unchanged.
This confirms that the contribution of $\mathrm{vec}(\mathbf{H})$ comes from structure-conditioned alignment rather than dimensionality alone.

\subsection{Support-set Criterion for Component Selection}
\label{app:component_selection}

Different datasets may benefit from different representation components.
To provide a practical criterion for component selection, we compute the within-class scatter $S_w$ and between-class scatter $S_b$ of each component on the few-shot support set using the frozen encoder.
We define a Fisher-style separability ratio as
\begin{equation}
    R = \frac{S_b}{S_w}.
\end{equation}
A larger $R$ indicates stronger class separability for the corresponding component.

Table~\ref{tab:fisher_ratio} reports the ratios of $\mathbf{w}$ and $\mathrm{vec}(\mathbf{H})$ on Cora, PROTEINS, and MUTAG.

\begin{table}[h]
\centering
\caption{Fisher-style separability ratio $R=S_b/S_w$ for different representation components.}
\label{tab:fisher_ratio}
  \begin{tabular}{l c c}
    \hline
    Dataset & $R$ of $\mathbf{w}$ / $\mathrm{vec}(\mathbf{H})$ & Dominant Component \\
    \hline
    Cora     & 58.1 / \textbf{231.4} & $\mathrm{vec}(\mathbf{H})$ \\
    PROTEINS & \textbf{304.3} / 71.7 & $\mathbf{w}$ \\
    MUTAG    & \textbf{143.5} / 105.6 & $\mathbf{w}$ \\
    \hline
  \end{tabular}
\end{table}

This provides a simple practical rule: for a new dataset, the component with the larger support-set ratio $R$ can be assigned higher importance.
The criterion also explains the observed behavior across datasets: Cora benefits more from the structure-conditioned feature representation, whereas PROTEINS and MUTAG benefit more from the structural coordinate representation.

\subsection{Quality of the SGW Approximation}
\label{app:sgw_approximation}

SCGFM adopts the Sliced Gromov--Wasserstein distance to approximate the exact GW distance for scalable training.
To evaluate the approximation quality, we compute exact GW and SGW distances on 500 graph pairs from PROTEINS.
The Pearson correlation between the two distances is
\[
    \rho = 0.7576, \qquad p < 10^{-90}.
\]
This is consistent with the NCI1 result in Figure~\ref{fig:isometry_study}, where the correlation is $\rho=0.7005$.

These results indicate that SGW preserves the relative geometric ordering of graph pairs while reducing the computational cost from approximately $O(N^3)$ for exact GW to $O(N\log N)$ for SGW-based computation.
This reduction is crucial for scalable pretraining across diverse graph datasets.

\subsection{Effect of the Linear Reconstruction Surrogate}
\label{app:linear_surrogate}

Exact GW barycenter computation requires nested OT iterations, which are computationally expensive and difficult to integrate into end-to-end gradient-based training.
SCGFM therefore uses a linear reconstruction surrogate as defined in Eq.~\ref{eq:surrogate_B}, enabling efficient differentiable reconstruction from structural coordinates.

We conduct a diagnostic experiment on PROTEINS over 200 epochs.
Table~\ref{tab:linear_surrogate} compares the exact barycenter reconstruction and the linear surrogate.

\begin{table}[h]
\centering
\caption{Comparison between exact GW barycenter reconstruction and the linear surrogate on PROTEINS.}
\label{tab:linear_surrogate}
  \begin{tabular}{l c c}
    \hline
    Metric & Exact Barycenter & Linear Surrogate \\
    \hline
    Avg. SGW recon. error & 0.018739 & 0.018776 $(+0.000037)$ \\
    Inference time / graph & 1.47 ms & \textbf{0.04 ms} $(36{\times})$ \\
    End-to-end trainable & No & \textbf{Yes} \\
    \hline
  \end{tabular}
\end{table}

The linear surrogate yields almost identical reconstruction fidelity, with only a negligible increase in SGW reconstruction error.
Meanwhile, it reduces inference time from 1.47 ms to 0.04 ms per graph and supports end-to-end training.
Therefore, the surrogate provides a practical and scalable approximation to GW barycentric reconstruction.

\subsection{Sensitivity to Sampling Strategy and Subgraph Size}
\label{app:sampling_sensitivity}

For node-level transfer, each target node is represented by a sampled ego-subgraph.
We study the sensitivity of SCGFM to both the sampling strategy and the sampled subgraph size on the CiteSeer $\rightarrow$ Cora 5-shot transfer setting.
We compare $k$-hop sampling, random walk (RW) sampling, and personalized PageRank (PPR) sampling.

Table~\ref{tab:sampling_size} reports the results under different subgraph sizes.
Performance generally improves when the subgraph size increases from 16 to 64 and then saturates around 64--128 nodes.

\begin{table}[h]
\centering
\caption{Sensitivity of SCGFM to sampling strategy and subgraph size on CiteSeer $\rightarrow$ Cora under the 5-shot setting.}
\label{tab:sampling_size}
  \begin{tabular}{l c c c}
    \hline
    Strategy & Size = 16 & Size = 64 & Size = 128 \\
    \hline
    $k$-hop & $62.97 \pm 3.70$ & $67.06 \pm 3.05$ & $65.93 \pm 2.82$ \\
    RW      & $59.17 \pm 3.98$ & $66.84 \pm 3.87$ & $69.67 \pm 2.87$ \\
    PPR     & $\textbf{64.73} \pm 3.53$ & $\textbf{71.23} \pm 2.73$ & $\textbf{71.39} \pm 2.55$ \\
    \hline
  \end{tabular}
\end{table}

We further compare sampling strategies across different models.
For each model and strategy, Table~\ref{tab:sampling_models} reports the best accuracy over the tested subgraph sizes.

\begin{table}[h]
\centering
\caption{Comparison of sampling strategies across models on CiteSeer $\rightarrow$ Cora. Each entry reports the best accuracy over tested subgraph sizes.}
\label{tab:sampling_models}
  \begin{tabular}{l c c c}
    \hline
    Model & $k$-hop & RW & PPR \\
    \hline
    SCGFM & $67.34@96$ & $69.67@128$ & $\textbf{71.39@128}$ \\
    GCN   & $\textbf{45.53@128}$ & $38.99@128$ & $44.44@128$ \\
    GIN   & $\textbf{45.86@96}$  & $38.73@128$ & $42.87@128$ \\
    \hline
  \end{tabular}
\end{table}

The results suggest two observations.
First, SCGFM benefits from moderately sized subgraphs, with performance saturating around 64--128 nodes.
Second, the preferred sampling strategy is model-dependent.
SCGFM achieves the best performance with PPR sampling, which better preserves globally relevant structural context, while GCN and GIN obtain their best results with $k$-hop neighborhoods.

\section{Proof}
\subsection{Proof of Theorem \ref{the:stability}} \label{proof:w_stability}
\begin{proof}
  Let $G$  and $G'$ be two input graphs with structural distance $d_{GW}(G,G')$.

  \textbf{Stability of Structural Coordinates ($\mathbf{w}$)}

  The coordinates are derived from Softmax-normalized distances to fixed bases.
  It has been established that the Softmax function is Lipschitz continuous with a constant of $L_{sm}$ with respect to the standard Euclidean norm \cite{gao2017properties}.
  Combined with the scaling factor $1/\tau$, we have:
  \begin{align}
    \Vert \mathbf{w} - \mathbf{w}' \Vert_2 
    &= \Vert \operatorname{softmax}(-\mathbf{d}/\tau) - \operatorname{softmax}(-\mathbf{d}'/\tau) \Vert \\ 
    &\leq L_{sm} \cdot \Vert -\mathbf{d}/\tau -(-\mathbf{d}' /\tau)\Vert \\
    &= \frac{L_{sm}}{\tau} \Vert \mathbf{d} - \mathbf{d}' \Vert_2 \label{eq:dd}
  \end{align}

  The vector $\mathbf{d}(G) \in \mathbb{R}^K$ consists of the GW distances from graph $G$ to the set of fixed geometric bases $\mathcal{B} = \{ B_1,B_2, \ldots, B_K \}$, i.e., $d_k = d_{GW}(G,B_k)$.
  
  Since the GW distance $d_{GW}$ satisfies the triangle inequality on the considered graph space, it also satisfies the reverse triangle inequality:
  \begin{equation}
    | d_{GW}(G,B_k) - d_{GW}(G', B_k)| \leq d_{GW}(G,G')
  \end{equation}
  for any basis $B_k$.

  Applying this inequality component-wise to the $l_2$-norm formulation of $\Vert\mathbf{d} - \mathbf{d}' \Vert_2$, we obtain:
  \begin{align}
    \Vert \mathbf{d} - \mathbf{d}' \Vert_2
    &= \sqrt{\sum_{k=1}^K ( d_{GW}(G,B_k) - d_{GW}(G',B_k))^2}\\
    & \leq \sqrt{\sum_{k=1}^K (d_{GW}(G,G'))^2} \\
    &= \sqrt{K(d_{GW}(G,G'))^2}\\
    &= \sqrt{K} \cdot d_{GW}(G,G')
  \end{align}

  Combining this result with Eq.~\ref{eq:dd}, we obtain the final Lipschitz bound for the weight vector $\mathbf{w}$:
  \begin{align}
    \Vert \mathbf{w} - \mathbf{w}' \Vert_2
    & \leq \frac{L_{sm}}{\tau}\Vert \mathbf{d} - \mathbf{d}' \Vert_2\\
    & \leq \frac{\sqrt{K} L_{sm}}{\tau}\cdot d_{GW}(G,G')
  \end{align}

  This confirms that the structural coordinates representation $\mathbf{w}$ is Lipschitz continuous with respect to the input graph structure, with a Lipschitz constant $L_{\mathbf{w}} = \frac{\sqrt{K}L_{sm}}{\tau}$.

\end{proof}


\end{document}